\documentclass{article}

\usepackage[nonatbib,preprint]{neurips_2021}

\usepackage[utf8]{inputenc} 
\usepackage[T1]{fontenc}    
\usepackage{hyperref}       
\usepackage{url}            
\usepackage{booktabs}       
\usepackage{amsfonts}       
\usepackage{nicefrac}       
\usepackage{microtype}      
\usepackage{etoc}

\usepackage{amsmath}
\usepackage{color}
\usepackage{amsfonts,enumitem}
\usepackage{amssymb,enumitem}

\usepackage{graphicx,tikz}
\usepackage{wrapfig}
\usepackage{multicol}

\usepackage{float}
\usepackage{mysty}
\usepackage{cleveref} 
\usepackage{subcaption} 

\newcommand{\sol}{\mathrm{sol}}

\newcommand{\relu}{\mathrm{relu}}
\newcommand{\diff}{\mathrm{diff}}

\newcommand{\conv}{\mathrm{conv}\,}

\newcommand{\jac}{\mathrm{Jac}\,}

\newcommand{\aff}{\mathrm{aff}}
\newcommand{\backprop}{\mathrm{backprop}}

\newcommand{\RR}{\mathbb{R}}
\newcommand{\NN}{\mathbb{N}}

\newcommand{\N}{\mathbb{N}}

\newcommand{\dic}{{\rm Dic}}

\newcommand{\supp}{\mathrm{supp}\,}

\newcommand{\partialc}{\partial^c}

\newcommand{\sign}{\mathrm{sign}}

\newtheorem{theorem}{Theorem}

\newtheorem{proposition}{Proposition}
\newtheorem{corollary}{Corollary}

\newtheorem{definition}{Definition}
\newtheorem{assumption}{Assumption}
\newtheorem{remark}{Remark}

\newtheorem{example}{Example}
\newenvironment{proof}[1][]{\noindent {\bf Proof#1:\;}}{\hfill $\Box$}

\newtheorem{innercustomthm}{Theorem}
\newenvironment{apptheorem}[1]
  {\renewcommand\theinnercustomthm{#1}\innercustomthm}
  {\endinnercustomthm}

\newtheorem{innercustomcor}{Corollary}
\newenvironment{appcor}[1]
  {\renewcommand\theinnercustomcor{#1}\innercustomcor}
  {\endinnercustomcor}
  
\newtheorem{innercustomprop}{Proposition}
\newenvironment{appprop}[1]
  {\renewcommand\theinnercustomprop{#1}\innercustomprop}
  {\endinnercustomprop}

\newcommand{\R}{\mathbb{R}}

\title{Nonsmooth Implicit Differentiation for Machine Learning and Optimization}

\author{
   J\'er\^ome Bolte\\
   Toulouse School\\ of Economics\\
   Univ. Toulouse\\
   Toulouse, France
   \And
   Tam Le\\
   Toulouse School\\ of Economics \\
   Univ. Toulouse\\
   Toulouse, France
   \And
   Edouard Pauwels\\
    IRIT, CNRS\\
   Univ. Toulouse\\
   Toulouse, France
   \And
   Antonio Silveti-Falls\\
   Toulouse School\\ of Economics \\
   Univ. Toulouse\\
   Toulouse, France
}

\begin{document}
\etocdepthtag.toc{mtsection}
\etocsettagdepth{mtsection}{subsection}
\etocsettagdepth{mtappendix}{none}
\maketitle

\begin{abstract}

    In view of training increasingly complex learning architectures, we establish a nonsmooth implicit function theorem with an operational calculus. Our result applies to most practical problems (i.e., definable problems) provided that a nonsmooth form of the classical invertibility condition is fulfilled. This approach allows for {\em formal subdifferentiation}: for instance, replacing derivatives by Clarke Jacobians in the usual differentiation formulas is fully justified for a wide class of nonsmooth problems. Moreover this calculus is entirely compatible with  algorithmic differentiation (e.g., backpropagation). We provide several applications such as training deep equilibrium networks, training neural nets with conic optimization layers, or hyperparameter-tuning for nonsmooth Lasso-type models. To show the sharpness of our assumptions, we present numerical experiments showcasing the extremely pathological gradient dynamics one can encounter when applying implicit algorithmic differentiation without any hypothesis.
    
\end{abstract}

\section{Introduction}

\paragraph{Differentiable programming.} 
The recent introduction of deep equilibrium networks \cite{bai2019equilibrium}, the increasing importance of bilevel programming (e.g., hyperparameter optimization) \cite{pedregosa2016hyperparameter} and the ubiquity of differentiable programming (e.g., TensorFlow \cite{45381}, PyTorch \cite{NEURIPS2019_9015}, JAX \cite{jax2018github}) in modern optimization call for the development of a versatile theory of nonsmooth differentiation. Our focus is on nonsmooth implicit differentiation. There are currently two practices lying at the crossroads of mathematics and computer science: on the one hand the use of the standard smooth implicit function theorem ``almost everywhere'' \cite{gould2021deep,elghaoui2019} and on the other hand the development of algorithmic differentiation tools  \cite{Agrawal2019differentiable,agrawal2019differentiating,kolter2020tutorial}. The empirical use of the latter in the nonsmooth world has shown surprisingly  efficient results \cite{kolter2020tutorial}, but the current theories cannot explain this success. We bridge this gap by providing nonsmooth implicit differentiation results and  illustrating their impact on the training of neural networks and hyperparameter optimization.

\paragraph{Backpropagation: a formal differentiation approach.} Let us consider $z$ implicitly defined through  $F(z(x))=h(x)$ where $F$ and $h$ have full domain and adequate dimensions. 
How  does  autograd apply to evaluating the ``derivative'' of the implicitly defined function $z$? Regardless of differentiability or nonsmoothness, and provided that inversion is possible, one commonly uses (or dynamically approximates) this derivative by $$\left(\backprop_{F}(z(x))\right)^{-1}\backprop_{h}x,$$
where $\backprop$ outputs the result of formal backpropagation, see e.g., \cite{Rumelhart:1986we}. 
This identity\footnote{The notation $\backprop_{z}$ instead of $\backprop (z)$ is indicative of the fact that $\backprop$ is an operator that does not act on functions themselves but rather on the program used to represent them, see \cite{bolte2020mathematical}.}  is used to provide efficient training despite the fact that the rules of classical nonsmooth calculus are transgressed \cite{bai2019equilibrium,kolter2020tutorial}. Note that  spurious outputs may be created by this approach, but on a negligible set. Consider for example the simple implicit problem $x = f(z(x))$ where $f(z) :=  \tanh(z) + \relu(-z) + z -\relu(z)$, whose solution is  $z(x)=\tanh x$. Yet 

applying the implicit differentiation framework of  \cite{bai2019equilibrium} using \texttt{JAX} library, as presented in \cite{kolter2020tutorial}, 
provides inconsistency of the derivative at the origin, see Figure~\ref{fig:ennemies}.
As mentioned above, despite these unpredictable outputs,  propagating derivatives leads to an  undeniable efficiency. But can we parallel these propagation ideas with a simple mathematical counterpart? Is there a rigorous theory backing up {\em formal (sub)differentiation} or  {\em formal propagation}? The answer is positive and was initiated in \cite{bolte2020conservative,bolte2020mathematical} through conservative Jacobians (see also \cite{lewis2021structure,davis2021conservative}).
\paragraph{A mathematical model for propagating derivatives.}
Conservative calculus models nonsmooth algorithmic differentiation faithfully and allows for a sharp study of training methods in Deep Learning \cite{bolte2020conservative,bolte2020mathematical}. It involves a new class of derivatives, generalizing Clarke Jacobians \cite{clarke1983optimization}. A distinctive feature of conservative calculus is that it is preserved by Jacobian multiplication. Consider for example a feed forward network combining analytic or relu activations and max pooling. A conservative Jacobian for this network can be obtained by using Clarke Jacobians formally as classical Jacobians, regardless of qualification conditions. For instance, Figure~\ref{fig:ennemies} depicts a selection in a conservative Jacobian.
This approach is general enough to handle spurious points such as in Figure ~\ref{fig:ennemies} while keeping the essence of the properties one expects from a derivative. It was proved in \cite{bolte2020conservative} that $\backprop$, applied to any reasonable program of a function, is a conservative Jacobian for this function; in contrast, $\backprop$ cannot be modelled by some subdifferential operator. For instance for the fixed point problem above, given conservative Jacobians $J_F$ and $J_h$ (e.g., Clarke Jacobians) for $F$ and $h$ one obtains a new conservative Jacobian $J_z$ implicitly defined through
\begin{align*}
    J_{F}(z(x))J_z(x)=J_h (x).
\end{align*}
 This property exactly parallels the idea of  ``propagating derivatives'' in practice. It gives a strong meaning to the formal use of Jacobians proposed in \cite{bai2019equilibrium}, and many empirical approaches \cite{gu2020implicit,Agrawal2019differentiable,gould2021deep,elghaoui2019}.

\paragraph{Main contributions:} $\:$\\
\textbf{---} We establish a nonsmooth conservative implicit function theorem that comes with an {\em implicit calculus} which is the central focus of this paper. Our calculus amounts somehow to {\em formal subdifferentiation with Clarke Jacobians}. This approach cannot rely on classical tools like the inverse of a Clarke Jacobian or a composition of Clarke Jacobians, which are not in general Clarke Jacobians. Indeed, a surprising example (Example~\ref{ex:counterExClarkeInverse}) shows that an ``inverse function theorem with Clarke calculus'' is not possible. 

\textbf{---} We study a wide range of applications of our implicit differentiation theorem, covering deep equilibrium problems \cite{bai2019equilibrium}, conic optimization layers \cite{Agrawal2019differentiable}, and hyperparameter optimization for the Lasso \cite{bertrand2020implicit}. Each case is detailed and its specificities are discussed.

\textbf{---} As a consequence, we obtain convergence guarantees for mini-batched stochastic algorithms with vanishing step size for training wide classes of Neural Nets, or for Lasso hyperparameter selection. The assumptions needed for our results are mild and fulfilled by most losses occurring in ML in the spirit of \cite{bolte2020mathematical,lee2020correctness}: elementary log-exp functions \cite{bolte2020mathematical}, semialgebraic functions \cite{bochnak2013real}, all being subclasses of definable functions \cite{coste2000introduction,van1996geometric}. The use of such structural classes has become standard in nonsmooth optimization and is more and more common in ML (see, e.g.,  \cite{castera2019inertial,bolte2020mathematical,lee2020correctness,ji2020directional}).

\textbf{---} As in the smooth implicit function theorem, the invertibility condition is not avoidable in general. We provide various examples for which the assumption is not satisfied; this results in severe failures for the corresponding gradient methods. In Figure~\ref{fig:ennemies}, one sees how lack of invertibility on an otherwise ordinary problem may provide totally unpredictable behavior for smooth quadratic optimization.

\begin{figure}[ht]
\centering
    \includegraphics[width=.3\textwidth]{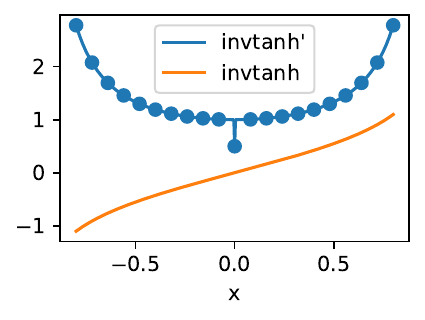} \qquad
     \includegraphics[width=.27\linewidth]{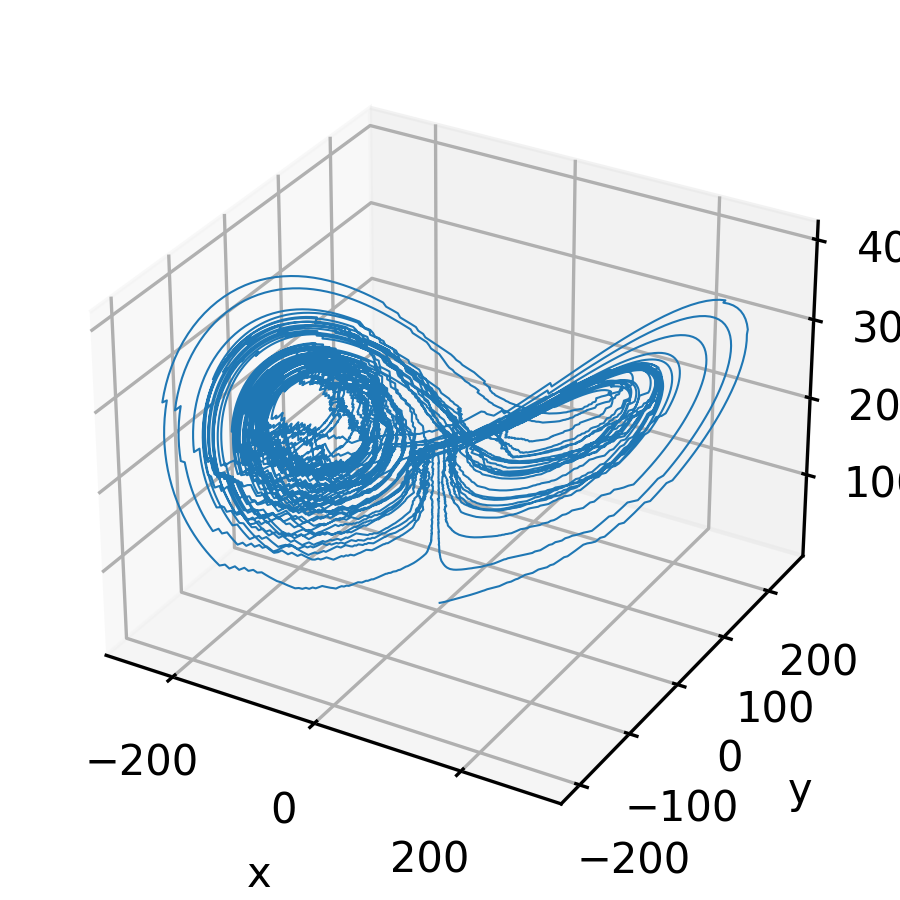}
    \caption{ Left: Inconsistencies due to combination of implicit differentiation and algorithmic differentiation. Right: A gradient  trajectory of an implicitly defined quadratic function.}
    \label{fig:ennemies}
\end{figure}

\paragraph{Definitions and Notations.}
A function $F:\R^n\to\R^m$ is {\em locally Lipschitz} if, for each $x\in\R^n$, there exists a neighborhood $\mcU$ of $x$ such that $F$ is Lipschitz on $\mcU$. Given matrices $A\in\R^{n\times m}$ and $B\in\R^{n\times p}$,  $[A\ B]\in\R^{n\times \para{m+p}}$ denotes their concatenation; $\Id_n$ denotes the $n\times n$ identity matrix. For $q\in\R^n$, $\diag\para{q}\in\R^{n\times n}$ denotes the diagonal matrix whose diagonal entries are given by the $q_i$; $\sign\para{q}\in\brac{-1,0,1}^n$ denotes the componentwise sign function. The {\em convex hull} of $\mcU$ is denoted  $\conv\, \mcU$.   The projection onto a closed convex set $\mcC\in\R^n$ is given, for each $x\in\R^n$, by $\pproj_{\mcC}\para{x} := \argmin\{\frac{1}{2}\norm{u-x}{}^2:u\in\mcC\}$. 
Given a convex proper lower semicontinuous function $f:\R^n\to\R\cup\brac{\pinfty}$, we define its proximal operator through $x\in\R^n$,
$ 
\prox_f \para{x}  :=  \argmin \{f\para{u} + \frac{1}{2}\norm{u-x}{}^2u\in\R^n\}.$ Set-valued maps are denoted by $\rightrightarrows$, for example the subgradient $\partial f \colon \RR^n \rightrightarrows \RR^n$. Additional details and notations are provided in Appendix~\ref{app:lexicon}.

\section{Implicit Differentiation with Conservative Jacobians}
\label{sec:conservativeJacobians}

\paragraph{Definitions and conservativity.} Conservative Jacobians are generalized forms of Jacobians well suited for automatic differentiation, introduced in \cite{bolte2020conservative}. Given a locally Lipschitz continuous function $F:\R^n\to \R^m$, we say that $J_F:\R^n\rightrightarrows \R^{n\times m}$ is a {\em conservative mapping} or a {\em conservative Jacobian} for $F$ if $J_F$ has a closed graph, is locally bounded, and is nonempty with 
\begin{equation}\label{eq:chainpath}
   \frac{\mathrm{d}}{\mathrm{d}t}F(\gamma(t)) = J_F(\gamma(t)) \dgamma (t) \mbox{ a.e.} 
\end{equation}
whenever $\gamma$ is an absolutely continuous curve in $\R^n$. When $m=1$, the corresponding vectors are called {\em conservative gradient fields}.  Note that when $J_F$ is conservative, so is its pointwise convexified extension $\conv J_F$.

A locally Lipschitz function is called {\em path differentiable} if it has a conservative Jacobian. 
 Recall that the {\em Clarke Jacobian} is defined as 
\begin{align*}
    \cjac F(x) = \operatorname{conv} \left\{ \lim_{k \to +\infty} \jac F(x_k) : x_k \in \diff_F, x_k \underset{k \to +\infty}{\xrightarrow[]{}} x\right\}
\end{align*}
where $\diff_F$ is the full measure set of points where $F$ is differentiable and $\jac F$ is the standard Jacobian of $F$. Path differentiability is equivalent to having a chain rule as in \eqref{eq:chainpath} for the Clarke subdifferential, see \cite{bolte2020conservative,davis2020stochastic}.

\paragraph{Examples of path differentiable functions and conservative Jacobians.}
(a) Convex functions and concave functions are path differentiable, see \cite{bolte2020conservative}. This implies that their subdifferential in the sense of convex analysis is a conservative field. \\
(b) The vast class of definable functions are path differentiable \cite{davis2020stochastic,bolte2020conservative}. As a result, the Clarke Jacobian of a Lipschitz definable mapping is a conservative Jacobian. Definable functions (see \cite{attouch2010proximal,davis2020stochastic, castera2019inertial, bolte2020conservative} for an optimization context and \cite{coste2000introduction} for a foundational work) encompass semialgebraic functions \cite{bochnak2013real}, elementary log-exp selection \cite{bolte2020conservative}, PAP \cite{lee2020correctness} (restricted to analytic functions with full domain), and many others, see \cite{van1996geometric} and references therein. This includes networks with common nonlinearities: for example analytic with full domain (e.g., square, exponential, logistic loss, hyperbolic tangent, sigmoid), relu, max pooling, sort, (see Appendix~\ref{app:definable} for more detail). \\
(c) The backpropagation can be seen as an oracle (in the optimization sense) for a conservative Jacobian. Let $P_F$ be a numerical program for a function $F$, aggregating elementary functions, for instance, relu, max pooling, affine mappings, polynomials (in general, any definable function). 
Then the backpropagation algorithm applied to $P_F$, which we denote (abusively) by $\backprop\, P_F:= \backprop_F$, outputs an element of a conservative Jacobian  \cite[Theorem 8]{bolte2020conservative} which depends on $P_F$ and can be constructed by a closure procedure \cite[definition 5]{bolte2020mathematical}.  As described in \cite{bolte2020mathematical}, due to spurious behaviors, $\backprop_F$ is not in general an element of the Clarke Jacobian of $F$.

\paragraph{The structure of conservative Jacobians.}
As established in \cite{lewis2021structure} in a semialgebraic context, the discrepancy between conservative gradients and Clarke subdifferentials is somehow negligible. Let us provide a version of that result matching our concerns.  We call conservative mappings of the null function {\em residual} or {\em residual conservative}. Such a mapping $R$ has the property that $R(x+tv)v=0$ for almost all $t$ in $\R$ and all $x,v$ in $\R^n\times \R^n$. 
The following theorem and proposition (partially) extend results from \cite{bolte2020conservative} and \cite{lewis2021structure}, their proof is given in Appendix~\ref{app:conservativeJacobians}.
\begin{theorem}[The Clarke Jacobian is a minimal conservative Jacobian]	Given a nonempty open subset $\mathcal{U}$ of $\R^n$ and $F : \mathcal{U}\subset \R^n \to \R^m$ locally Lipschitz, let $J_F$ be a convex-valued conservative Jacobian for $F$. Then for almost all $x \in \mathcal{U}, J_F(x) = \left\{\jac F \right\}$ and, for all $x \in \mathcal{U}$, $\jac^c F(x) \subset J_F(x)$. 
	\label{th:incluseionClarke}
\end{theorem}
\begin{proposition}[Decomposition of conservative fields]\label{prop:conservfieldsdecomp}
Let $J_F$ be a conservative Jacobian for $F$, then there is a residual $R$ such that
\begin{align*}
    J_F\subset \jac^c F +R.
\end{align*}
\end{proposition}
Note that the above may not hold with equality. Consider $F(x)=|x|$ and $J_F(0)=[-1,1]\cup [2,3]$, $J_F(x)=\sign\, (x)$ otherwise. One cannot write $J_F=\jac^c F+R$ with a residual operator $R$.

\paragraph{Formal subdifferentiation in a nonsmooth setting.} Propagating derivatives within a nonsmooth function finds  its justification in the following:
\begin{proposition}[Stability by composition, \cite{bolte2020conservative}]
\label{prop:jacobianComposition}
Let $F:\R^n\to\R^m$ and $G:\R^m\to \R^l$ be two locally path differentiable functions having respective conservative Jacobians $J_F$ and $J_G$. Then $F\circ G$ is path differentiable and the point-to-set matrix-valued $x\rightrightarrows J_F(G(x)) J_G(x)$ is conservative.
\end{proposition}

\paragraph{A conservative Implicit Function Theorem.} There is already a long tradition of nonsmooth implicit function theorems, e.g.,  \cite{clarke1983optimization,AUBIN1987,robinson1991implicit,dontchev2009implicit}. What makes the following theorem useful is that it comes with a qualification-free calculus. The proofs are given in  Appendix~\ref{app:conservativeJacobians}.
\begin{theorem}[Implicit differentiation]\label{th:implicitMainTheorem}
Let $F:\R^n\times\R^m\to \R^m$ be path differentiable on $\mcU\times\mcV\subset\R^n\times \R^m$ an open set and $G:\mcU\to\mcV$ a locally Lipschitz function such that, for each $x\in \mcU$,
\nnewq{\label{eq:Gsolves}
F(x,G(x))=0.
}
Furthermore, assume that for each $x\in \mcU$, for each $[A\ B]\in \cmap_F(x,G(x))$, the matrix $B$ is invertible where $\cmap_F$ is a conservative Jacobian for $F$. Then, $G:\mcU\to\mcV$ is path differentiable with conservative Jacobian given, for each $x\in\mcU$, by
\newq{
\cmap_{G}\colon x \rightrightarrows  \left\{-B^{-1}A : [A\ B]\in \cmap_F(x, G(x))\right\}.
}
\end{theorem}

\begin{corollary}[Path differentiable implicit function theorem]\label{cor:implicitFun}
Let $F:\R^n\times\R^m\to \R^m$ be path differentiable with conservative Jacobian $J_F$. 
Let $(\hat{x},\hat{y})\in\R^n\times \R^m$ be such that $F(\hat{x},\hat{y})=0$. Assume that $\cmap_F(\hat{x},\hat{y})$ is convex and that, for each $[A\ B]\in \cmap_F(\hat{x},\hat{y})$, the matrix $B$ is invertible. Then, there exists an open neighborhood $\mcU\times\mcV\subset\R^n\times\R^m$ of $\para{\hat{x},\hat{y}}$ and a path differentiable function $G:\mcU\to\mcV$ such that the conclusion of Theorem~\ref{th:implicitMainTheorem} holds.
\end{corollary}

\begin{corollary}[Path differentiable inverse function theorem]\label{cor:inverse function}
Let $\mcU$ and $\mcV$ be open neighborhoods of $0$ in $\R^n$ and $\Phi:\mcU \to \mcV$ path differentiable  with $\Phi(0)=0$. Assume that $\Phi$ has a conservative Jacobian $\cmap_\Phi$ such that $\cmap_\Phi(0)$ contains only invertible matrices. Then, locally, $\Phi$ has a path differentiable inverse $\Psi$ with a conservative Jacobian given by
$$\cmap_\Psi(y)=\left\{A^{-1}: A\in J_{\Phi}(\Psi(y))\right\}.$$
\end{corollary}

\begin{remark}{\rm (a) {\bf(On the necessity of conservativity)} Example \ref{ex:counterExClarkeInverse} in Appendix~\ref{app:conservativeJacobians} shows that one cannot hope for the formulas in Corollaries \ref{cor:implicitFun} \& \ref{cor:inverse function} to provide Clarke Jacobians in general, even if the input(s)  are Clarke Jacobians themselves.\\
(b) {(\bf Lipschitz definable implicit and inverse function theorems)}  See Theorem \ref{th:invdef} and \ref{th:inpdef} in the appendix}
\end{remark}

\section{Nonsmooth implicit differentiation in Machine Learning}
\label{sec:examplesML}
Detailed proof arguments for all considered models are given in Appendix~\ref{app:examplesML}.

\paragraph{Monotone deep equilibrium networks.} Deep Equilibrium Networks (DEQs) \cite{bai2019equilibrium} are specific neural network architectures including layers whose input-output relation is implicitly defined through a fixed point equation of the form
\begin{align}
    z = f(z,x)
    \label{eq:implicitLayer}
\end{align}
where $x \in \RR^p$ is a given input and $z \in \RR^m$ is the corresponding output. We may consider that the variable $x$ represents both the  input layer and layer parameters. Assuming that, for each $x \in \RR^p$, there is a unique $z \in \RR^m$ satisfying the relation \eqref{eq:implicitLayer}, this defines an input-output relation $z \colon \RR^p \to \RR^m$. Furthermore, if $f$ is path differentiable with convex-valued conservative Jacobian $\cmap_f \colon \RR^m \times \RR^p \rightrightarrows \RR^{m \times (m+p)}$ whose projection on the first $m$ columns are all invertible, then the function $z$ itself admits a conservative Jacobian which can be computed from Theorem~\ref{th:implicitMainTheorem}.

We now focus on monotone operator implicit layers \cite{winston2020monotone} for which assumptions are easily stated. Our method applies to other similar architectures, e.g., DEQs \cite{bai2019equilibrium} or implicit graph neural networks \cite{gu2020implicit}. 
Let $\sigma \colon \RR^m \to \RR^m$ be the proximal operator of a convex function and assume $\sigma$ is path differentiable with conservative Jacobian $\cmap_\sigma \colon \RR^m \rightrightarrows \RR^{m \times m}$, assumed to be convex-valued. This encompasses the majority of activation functions used in practice \cite{combettes2020deep}.
Let  $W \in \RR^{m \times m}$ be a matrix such that $W + W^T \succeq 2\theta I$ with $\theta>0$. Under these  assumptions the implicit equation
    \begin{align}\label{eq:monotonelayer}
        z = \sigma(Wz + b)
    \end{align}
has a unique output $z(W,b)$ \cite[Theorem 2]{winston2020monotone}. The transformation $(W,b)\mapsto z(W,b)$ is  a {\em monotone implicit layer.}

The set-valued mapping obtained from Theorem~\ref{th:implicitMainTheorem} provides a  conservative Jacobian for $(W,z) \mapsto z(W,z)$. A similar expression was described in \cite[Theorem 2]{winston2020monotone}, without using conservativity and using the Clarke Jacobian formally as a classical Jacobian. The proposition below provides a full justification of this heuristic and ensures convergence of algorithmic differentiation based training.
\begin{proposition}[Path differentiation through monotone layers]    \label{prop:DEQ}
    Assume that $\cmap_\sigma$ is convex-valued and that, for all $J \in \cmap_\sigma(Wz(W,b)+b)$, the matrix $(\Id_m - JW)$ is invertible.  
Consider a loss-like function $\ell \colon \RR^m \to \RR$ with conservative gradient $D_\ell \colon \RR^m \rightrightarrows \RR^m$, then  $g:(W,z) \mapsto \ell(z(W,b))$ is path differentiable and has a conservative gradient $D_g$ defined through
    \begin{align*}
        D_g \colon (W,b) \rightrightarrows \left\{ J^T(\Id_m - JW)^{-T}v z^T,  J^T(\Id_m - JW)^{-T}v): \,\, J \in J_\sigma(Wz + b),\, v \in D_\ell(z)  \right\}.
    \end{align*}
\end{proposition}

\begin{remark}
{\rm    Convexity and invertibility assumptions are satisfied when $\cmap_\sigma$ is the Clarke Jacobian \cite{winston2020monotone}.
}
\end{remark}

\paragraph{Optimization layers: the conic program case.} Optimization layers in deep learning may take many forms; we consider here those based on conic programming \cite{busseti2019solution,agrawal2019differentiating,Agrawal2019differentiable,amos2017optnet}. We follow  \cite{agrawal2019differentiating}, simplifying the analysis by ignoring infeasability certificates, which correspond to the absence of a primal-dual solution \cite{busseti2019solution}, in line with the implementation described in \cite[Appendix B]{Agrawal2019differentiable}. Consider a conic problem (P) and its dual (D):
\begin{center}
\begin{tabular}{p{6.5cm}p{6.5cm}}
	{\begin{equation*}
			\begin{array}{lll}
				\text{(P)}
				&\inf  &c^T x\\
				&\text{subject to} &  Ax+s=b\\
				&  &s\in  \mathcal{K}
			\end{array}
	\end{equation*}}
	&
	{ \begin{equation}
			\label{eq:primalDual}
			\begin{array}{lll}
				\text{(D)}&\inf& b^T y\\
				&\text{subject to}& A^T y+c=0\\
				&&y\in  \mathcal{K}^*,
			\end{array}
	\end{equation}}
\end{tabular}
\end{center}
with primal variable $x\in \RR^n$, dual variable $y\in \RR^m$, and  primal slack variable $s\in \RR^m$. The set $\mathcal{K}\subset \RR^m$ is a nonempty closed convex cone and $\mathcal{K}^*\subset \RR^m$ is its dual cone. The problem parameters are the matrix $A \in \RR^{m \times n}$ and the vectors $b\in \RR^m $ and $c\in \RR^n$; the cone $ \mathcal{K}$ is fixed. Under the assumption  that there is a unique primal-dual solution $(x,y,s)$, we study the path differentiability of the solution mapping  as a function of its parameters:
$$ (A,b,c)\mapsto \sol(A,b,c)=(x,y,s).$$

For this, let us interpret the solution mapping as a composition mapping involving equation-like implicit formulations. Set $N=n+m$, given $A,b,c \in \RR^{m\times n}\times \RR^m\times \RR^n$, define 
\begin{align*}
    Q(A,b,c) = 
        \begin{bmatrix}
            0 & A^T \\
            -A & 0 
        \end{bmatrix}
    \in \RR^{N\times N} \qquad V(b,c) = 
        \begin{bmatrix}
            c  \\
            b 
        \end{bmatrix}  \in \RR^{N}.
\end{align*}
Consider a vector $z =(u,v) \in \RR^n \times \RR^m$, denote by $\proj$ the projection onto $\RR^n \times \mathcal{K}^*$ and define the {\em residual map} $ \mathcal{N}: \RR^N \times \RR^{m\times n}\times \RR^m\times \RR^n \to \RR^N$ as 
\begin{align*}
    \label{eq:residualMap}
    \mathcal{N}(z,A,b,c) =(Q(A,b,c)-\Id_N)\proj z +  V(b,c) + z.
\end{align*}
 The mapping $\mathcal N$ is a synthetic form of optimality measure for 
(P) and (D), capturing KKT conditions. To simplify the presentation, we ignore the extreme cases of infeasibility and unboundedness which correspond to an absence of solution in \cite{busseti2019solution}. 

Define the function $\phi : \RR^{N} \to \RR^n\times \RR^m\times \RR^m$ through 
$\displaystyle \phi(u,v) := (u, P_\mathcal{K^*}(v), P_{\mathcal{K}^*}(v)-v)$. As shown in Appendix~\ref{app:cone}, $\phi(u,v)$ provides a primal-dual KKT solution of problems (P) and (D) if and only if $\mathcal{N}(z,A,b,c) = 0$.
When we assume that, for fixed $A,b,$ and $c$, there is a unique $z \in \RR^N$ such that $\mathcal{N}(z,A,b,c) = 0$, we have an implicitly defined a function $ z=\nu(A,b,c)$, such that
\begin{equation}\label{eq:solconic}
\sol(A,b,c)=\left[\phi \circ \nu \right](A,b,c).
\end{equation}

The following result extends the discussion in \cite{busseti2019solution,agrawal2019differentiating}, limited to situations where $\Pi$ is differentiable at the proposed solution $z$, to a fully nonsmooth setting; its proof is postponed to  Appendix~\ref{app:cone}.

\begin{proposition}[Path differentiation through cone programming layers]
\label{prop:pathdiffconeprog}
Assume that $P_\mathcal{K^*}$, $\mathcal{N}$ are path differentiable,  denote respectively by $\cmap_{P_\mathcal{K^*}}$, $\cmap_\mathcal{N}$ corresponding convex-valued conservative Jacobians. Assume that, for all $A,b,c \in \RR^{m\times n}\times \RR^m\times \RR^n$, $z = \nu(A,b,c) \in \RR^n \times \RR^m$ is the unique solution to $\mathcal{N}(z,A,b,c) = 0$ and that all matrices formed from the $N$ first  columns of $\cmap_\mathcal{N}(z, A, b, c)$ are invertible. 
Then, $\phi$, $\nu$, and $\sol$ are path differentiable functions with conservative Jacobians:
    \begin{align*}
&     \cmap_{\nu}(A,b,c) :=  \left\{-U^{-1}V: [U\ V]\in \cmap_{\mathcal{N}}(\nu(A,b,c),A,b,c)\right\},\\
& \cmap_\phi(z) := \begin{bmatrix}
        \Id_n & 0  \\
        0&\cmap_{P_{\mathcal{K}^*}}(v)  \\
        0 &(\cmap_{P_{\mathcal{K}^*}}(v) - \Id_m)  
    \end{bmatrix},\\
& \cmap_{\sol}(A,b,c) :=  \cmap_\phi(\nu(A,b,c)) \cmap_\nu(A,b,c).
\end{align*}

       \label{cor:optLayers}
\end{proposition}

In practice, the path differentiability of conic projections is pervasive since they are generally semialgebraic (orthant, second-order cone, PSD cone). See \cite{hayashi2005combined,malick2006clarke,kong2009clarke,malick2006clarke} for the computations of the corresponding Clarke Jacobians (which are conservative). Note that a conservative Jacobian for $\mathcal{N}$ may be obtained from $\cmap_{P_{\mathcal{K}^*}}$ using Proposition~\ref{prop:jacobianComposition}.

\paragraph{Hyperparameter selection for Lasso type problems.} Implicit differentiation can be used to tune hyperparameters via first-order methods optimizing some measure of task performance, see \cite{bertrand2021implicit} and references therein. In a nonsmooth context, we recall the formulation in \cite{bertrand2020implicit} of the general hyperparameter optimization problem as a bi-level optimization problem:
$$
\min\limits_{\lambda\in\R^m} C(\bhat(\lambda)) \quad\mbox{such that}\quad \bhat(\lambda) \in \argmin\limits_{\beta\in\R^p} \psi(\beta, \lambda)
$$
where $C:\R^p\to\R$ is continuously differentiable (e.g., test loss) and $\psi \colon \RR^p \times \RR^m \to \RR$ is a possibly nonsmooth training loss, convex in $\beta$, with hyperparameter $\lambda\in\R^m$. We seek a subgradient type method for this problem with convergence guaranties; our nonsmooth implicit differentiation results can be used for this purpose. We demonstrate this approach on the Lasso problem \cite{tibshirani1996regression}

\nnewq{\label{eq:lasso}
\bhat\left(\lambda\right)\in \argmin  \brac{\frac{1}{2}\norm{y-X\beta}{2}^2 + e^{\lambda}\norm{\beta}{1}: \beta \in \R^p}
}
where $y\in\R^n$ is the vector of observations, $X = [X_1,\ldots,X_p]\in\R^{n\times p}$ is the design matrix with columns $X_j\in \R^n$, $j \in \{1 ,\ldots p\}$, and $\lambda\in\R$ is the hyperparameter. 
Define $F:\R\times\R^p\to\R^p$ to be
\newq{\label{eq:fixedPt}
F\para{\lambda,\beta} :=  \beta - \prox_{e^{\lambda}\norm{\cdot}{1}}\para{\beta-X^T\para{X\beta-y}}
}
and recall that, for each $i\in\{1,\ldots,p\}$, $[\prox_{e^{\lambda}\norm{\cdot}{1}}(\beta)]_i = 
\sign(\beta_i)\max\{|\beta_i|-e^{\lambda},0\}$. The function 
$F(\lambda,\beta)$ is thus nonsmooth but locally Lipschitz on $\R\times\R^p$. An optimal $\bhat(\lambda)$ for \eqref{eq:lasso} must satisfy $F(\lambda,\bhat(\lambda))=0$ \cite[Prop. 3.1]{combettes2005signal}. For a given solution $\bhat(\lambda)$, we introduce the equicorrelation set by $\mcE:=  \{j \in \brac{1,\ldots,p}: |X_{j}^T(y-X\bhat(\lambda))|=e^{\lambda}\}$ which contains the support set $\supp \hat \beta  :=  \{i\in\{1,\ldots,p\}:\bhat_i \neq 0\}$. In fact, $\mcE$ does not depend on the choice of the solution $\hat \beta$, see  \cite[Lemma 1]{tibshirani2013lasso}. The proof of the following result is given in Appendix~\ref{app:HO}.

\begin{proposition}[Conservative Jacobian for the solution mapping]\label{prop:HOprop}

For all $\lambda\in\R$, assume $X_{\mcE}^TX_{\mcE}$ is invertible where $X_{\mcE}$ is the submatrix of $X$ formed by taking the columns indexed by $\mcE$. Then $\bhat(\lambda)$ is single-valued, path differentiable with conservative Jacobian, $\cmap_{\bhat}\para{\lambda}$,  given for all $\lambda$ as
$$\left\{
\sbrac{-e^{\lambda}\para{\Id_p - \diag\para{q}\para{\Id_p - X^TX}}^{-1} \diag\para{q}\sign\para{\bhat - X^T\para{X\bhat-y}}}\ : q \in \mathcal{M}(\lambda)\right\}
$$
where $\mathcal{M}(\lambda) \subset \RR^p$ is the set of vectors $q$ such that $q_i = 1$ if $i \in \supp \hat \beta$, $q_i = 0$ if $i \not \in \mcE$ and $q_i \in [0,1]$ if $i \in \mcE \setminus \supp \hat\beta$.
\end{proposition}
Taking, in Proposition~\ref{prop:HOprop}, $q_i=1$ for all $i\in\mcE$ corresponds to the directional derivative given by LARS algorithm \cite{efron2004least}, see also \cite{mairal2012complexity}. Alternatively, taking $q_i=0$ for $i\not\in\supp\bhat$ gives the weak derivative described by \cite{bertrand2020implicit}. Both are particular selections in $J_{\hat{\beta}}$, which is the underlying conservative field. 
\section{Optimizing implicit problems with gradient descent}
\label{sec:algorithms}

We establish the convergence of gradient descent algorithms for compositional learning  problems involving implicitly defined functions. The result follows from the previous section and the general convergence results of \cite{bolte2020mathematical}.

\paragraph{The minimization problem.} The applications considered in the previous section all yield minimization problems of the type
\begin{align}
    \min_{w \in \RR^p} \ell(w) :=  \frac{1}{N} \sum_{i=1}^N \ell_i(w)\,
 \mbox{  with } \,  \ell_i = g_{i,L} \circ g_{i,L-1} \circ \ldots \circ g_{i,1}\label{eq:finiteSum}
\end{align}
    where, for each $i \in\{1, \ldots, N\}$, $\ell_i:\RR^p\to\RR$ is a composition of functions having appropriate input and output dimensions. The indices $i$ correspond in practice to learning samples while the loss $\ell$ embodies an empirical expectation, as for instance in deep learning. We will enforce the following structural condition.

\begin{assumption}
    \label{ass:structure}
    For $i\in\{1,\ldots, N\}$ and $j \in \{1,\ldots,L\}$, the function $g_{i,j}$ is locally Lipschitz with conservative Jacobian $\cmap_{i,j}$ and one of the following holds
    
    $\bullet$\quad $g_{i,j}$ and $J_{i,j}$ are semialgebraic (or, more generally, definable).
    
    $\bullet$\quad $g_{i,j}$ is defined as $G$ in Theorem~\ref{th:implicitMainTheorem}, with  $F$ and $J_F$ semialgebraic (or, more generally, definable).
\end{assumption}
Actually, in Assumption~\ref{ass:structure} the second point implies the first point; we list both for clarity.
More details on semialgebraicity and definability are given in Appendix~\ref{app:definable}. Let us stress that virtually all elements entering the definition of neural networks are semialgebraic or, more generally, definable, see for example \cite{bolte2020mathematical} for a constructive model. In particular, beyond classical networks with usual nonlinearities (e.g., relu, sigmoid, max pooling \ldots), this setting encompasses (through Corollary~\ref{cor:implicitFun}):

\textbf{(a)} Deep equilibrium networks: each $g_{i,j}$ may correspond to usual explicit layers or an implicit layer involving a fixed point mapping and a learning sample $i$ as in \eqref{eq:monotonelayer} or \eqref{eq:implicitLayer}.

\textbf{(b)} Training with optimization layers: similarly, the inner maps $g_{i,j}$ may also be solution mapping to convex conic programs and related to the $\sol$ function  \eqref{eq:solconic} of conic problems.

\textbf{(c)} One may assume that $N=1, L=2$ and retrieve the hyperparameter tuning for Lasso in its implicit formulation.

\paragraph{SGD with backpropagation.} Algorithmic differentiation (AD) is an automated application of the chain rule of differential calculus. 
When  applied to $\ell_i$, it amounts to computing one element of the product $\cmap_i : = \prod_{j=1}^L \cmap_{i,j}$ by choosing one element in each $J_{i,j}$ with appropriate inputs given by intermediate results kept in memory during a forward computation of the composition.

In this context AD stochastic gradient descent requires an initial  $w_0 \in \RR^p$ and a sequence of \textit{i.i.d.} random indices uniform in $\{1,\ldots,N\}$, $(I_k)_{k \in \NN}$. It gives:
\begin{align}
    w_{k+1} &= w_k - s\alpha_k v_k\\
    v_k &\in \cmap_{I_k} (w_k), \qquad\qquad \text{(given by $\backprop$)},
    \label{eq:sgd}
\end{align}
where $(\alpha_k)_{k \in \NN}$ is a sequence of positive step sizes and $s \in (s_{\min},s_{\max})$ is a scaling factor where $s_{\max}>s_{\min}>0$. A simpler choice could be $v_k \in \partialc \ell_{I_k}(w_k)$, however, the chain rule used within algorithmic differentiation routines does  not produce subgradients (see, e.g., Figure~\ref{fig:ennemies}). In contrast, conservative Jacobians are faithful models of AD outputs. 
The asymptotic behavior of the above algorithm depends on the variational properties of the conservative Jacobian $\cmap:=  \frac{1}{N} \sum_{i=1}^N \cmap_i$.

\begin{theorem}[Convergence result]
    Consider  minimizing $\ell$ given in \eqref{eq:finiteSum} using algorithm \eqref{eq:sgd} under Assumption~\ref{ass:structure}. Assume furthermore the following
    \begin{itemize}
        \item\textbf{Step size:} $\sum_{k=1}^{+\infty} \alpha_k = +\infty$ and $\alpha_k = o(1/\log(k))$.
        \item\textbf{Boundedness:} there exists $M>0$, and $K \subset \RR^p$ open and bounded, such that, for all $s \in (s_{\min},s_{\max})$ and $w_0 \in \mathrm{cl}\  K$, $\|w_k\| \leq M$ almost surely.
    \end{itemize}
   For almost all $w_0 \in K$ and $s\in(s_{\min},s_{\max})$, the objective value $\ell(w_k)$ converges and all accumulation points $\bar{w}$ of $w_k$ are  Clarke-critical in the sense that $0 \in \partialc \ell(\bar{w})$.
    \label{th:convergenceAlgo}
\end{theorem}

This result shows that AD SGD may be applied successfully to all problems described in Section~\ref{sec:examplesML}, combining algorithmic differentiation with implicit differentiation. Its proof may be adapted directly from \cite{bolte2020conservative,bianchi2020convergence}; details are given in Appendix~\ref{app:Theorem3}.

\section{Numerical experiments}
\label{sec:pathologies}

Using implicit differentiation when the invertibility condition in Theorem~\ref{th:implicitMainTheorem} does not hold can result in absurd training dynamics.

\paragraph{A cyclic gradient dynamics via fixed-point/optimization layer.}

Consider the bilevel problem:
\begin{align}
    \label{eq:cvxcycledynamic}
    \min_{x,y,s} \quad &\ell(x,y,s) := (x-s_1)^2 + 4 (y-s_2)^2 \\
    \mathrm{s.t.} \quad& s \in s(x,y):=\arg \max \left\{ (a+b)( -3x +y + 2 ):\, a \in [0,3], b \in [0,5] \right\}. \nonumber
\end{align}
Problem~\eqref{eq:cvxcycledynamic} has an equivalent fixed-point formulation using projected gradient descent on the inner problem (Appendix~\ref{app:cycle_cvxopt_fixed_point}). Backpropagation applied to \eqref{eq:cvxcycledynamic} associates to $(x,y)$ the following:
\begin{equation}
    \label{eq:gradCycle}
     \nabla_{(x, y)} \ell(x, y, s(x)) +\tilde{J}_{s} (x,y)^T  \nabla_s \ell (x, y, s(x)) 
\end{equation}
where $\tilde{J}_s$ is piecewise derivative. 

We implement gradient descent for \eqref{eq:cvxcycledynamic}, evaluating \eqref{eq:gradCycle} either using \texttt{cvxpylayers} \cite{Agrawal2019differentiable} or the \texttt{JAX} tutorial \cite{kolter2020tutorial} for fixed-point layers. In both cases, the invertibility condition in Theorem~\ref{th:implicitMainTheorem} fails when $-3x+y+2=0$, resulting in discontinuity of $s$, affecting the dynamics globally: the gradient trajectory converges to a limit cycle of non critical points (Figure~\ref{fig:cycledynamic}); see Appendix~\ref{app:cycle} for details.

\textit{Persistence under small perturbations:} The limit cycle remains (Figure~\ref{fig:jittercycle}) when running the same experiment on a slightly perturbed version of Problem \eqref{eq:cvxcycledynamic} with perturbed initializations as well (Appendix~\ref{app:perturbedfunction}). 

\begin{figure}[h]
\centering
\begin{subfigure}{0.44\textwidth}
    \centering
    \includegraphics[width=0.7\textwidth]{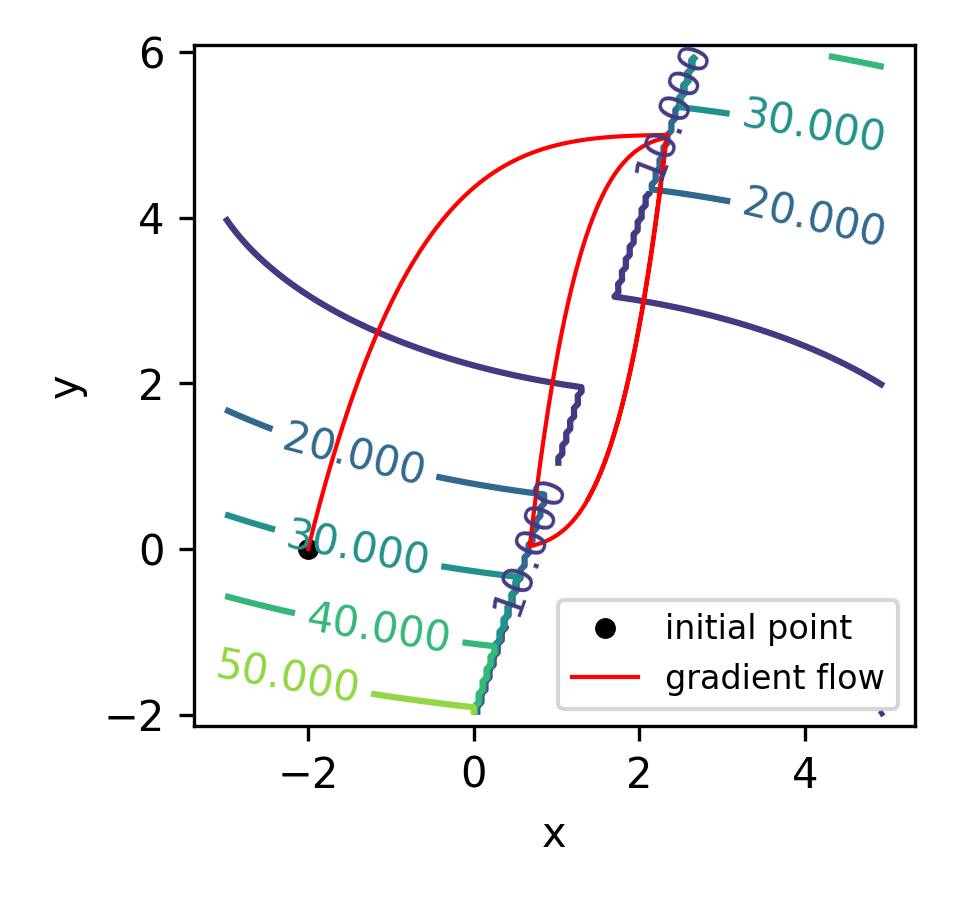}
    \caption{}
    \label{fig:cycledynamic}
\end{subfigure}%
\begin{subfigure}{0.44\textwidth}
    \centering
    \includegraphics[width=0.7\textwidth]{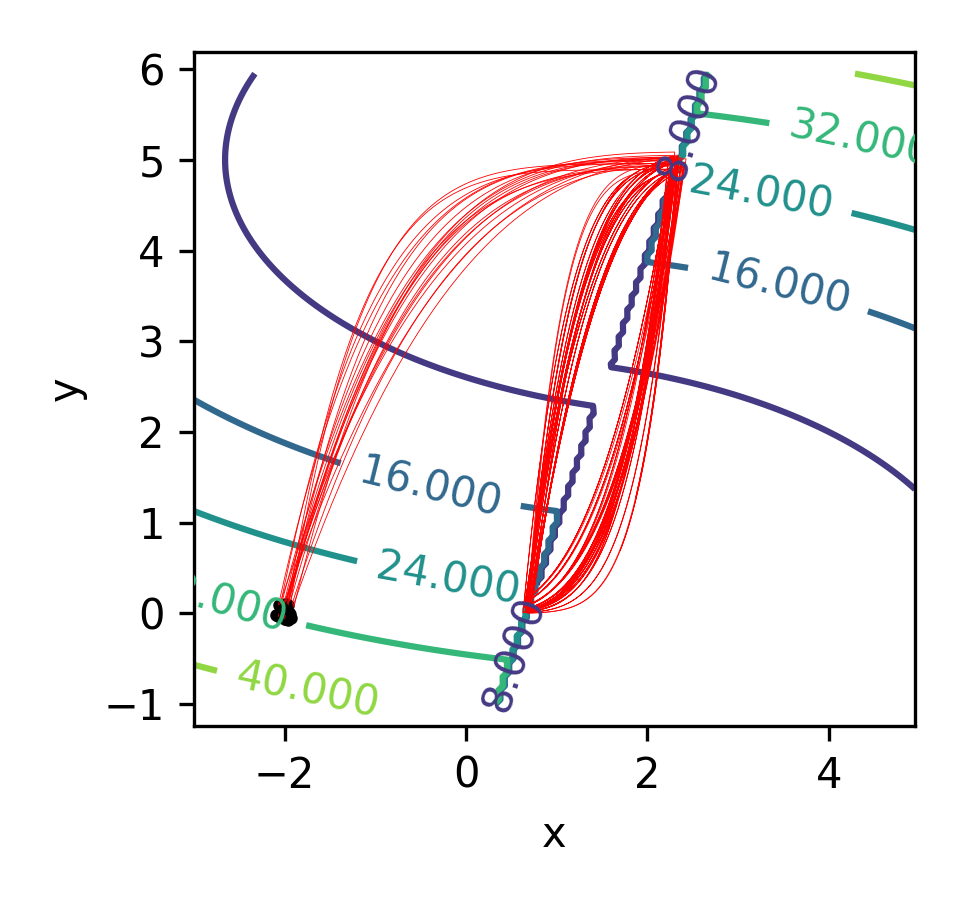}
    \caption{}
    \label{fig:jittercycle}
\end{subfigure}%
\caption{(a) Gradient flow for all implementations. (b) Gradient flows for 20 perturbed experiments.}
\end{figure}

\paragraph{A Lorenz-like dynamics via implicit differentiation.} The Lorenz Ordinary Differential Equation (ODE) writes:
\begin{equation}
\label{eq:lorenzsystem}
\begin{aligned}
    \dot{x} &= \sigma (y - x), \qquad    \dot{y} &= x (\rho - z) - y, \quad\mbox{and}\quad
    \dot{z} &= x y - \beta z.
\end{aligned}
\end{equation}
It is well-known that taking $(\sigma,\rho,\beta) = (10, 28,8/3)$, and $(x(0), y(0), z(0)) = (0, 1, 1.05)$ gives a chaotic trajectory, displayed in  Figure~\ref{fig:lorenzdynamic}. Denoting $F: (x, y, z) \mapsto (\sigma(y - x), x(\rho - z) - y, xy - \beta z)$ the vector field of the Lorenz system \eqref{eq:lorenzsystem}, consider the optimization problem:
\begin{align}
\label{eq:cvxpylayerlorenz}
    \max_{u \in \R^3} \quad u^T z \qquad \mathrm{s.t.} \qquad z \in \underset{s \in \R^3}{\arg\min} \|s - F(u)\|^4
\end{align}
which is obviously equivalent to
\begin{equation}
    \label{eq:pseudolorenzequivalentproblem}
    \max_{u \in \R^3} \quad u^T F(u).
\end{equation}
\begin{figure}[H]
\centering
\begin{subfigure}{.32\textwidth}
  \centering
  \includegraphics[width=1\linewidth]{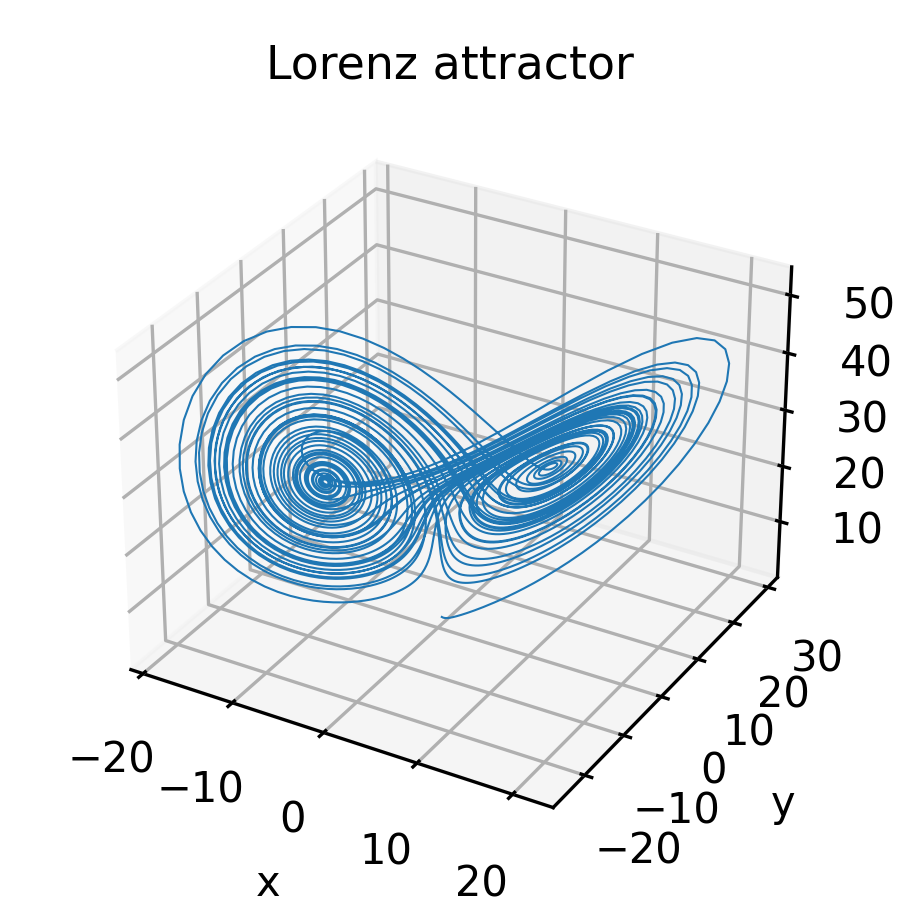}
  \caption{}
  \label{fig:lorenzdynamic}
\end{subfigure}%
\begin{subfigure}{.32\textwidth}
  \centering
  \includegraphics[width=1\linewidth]{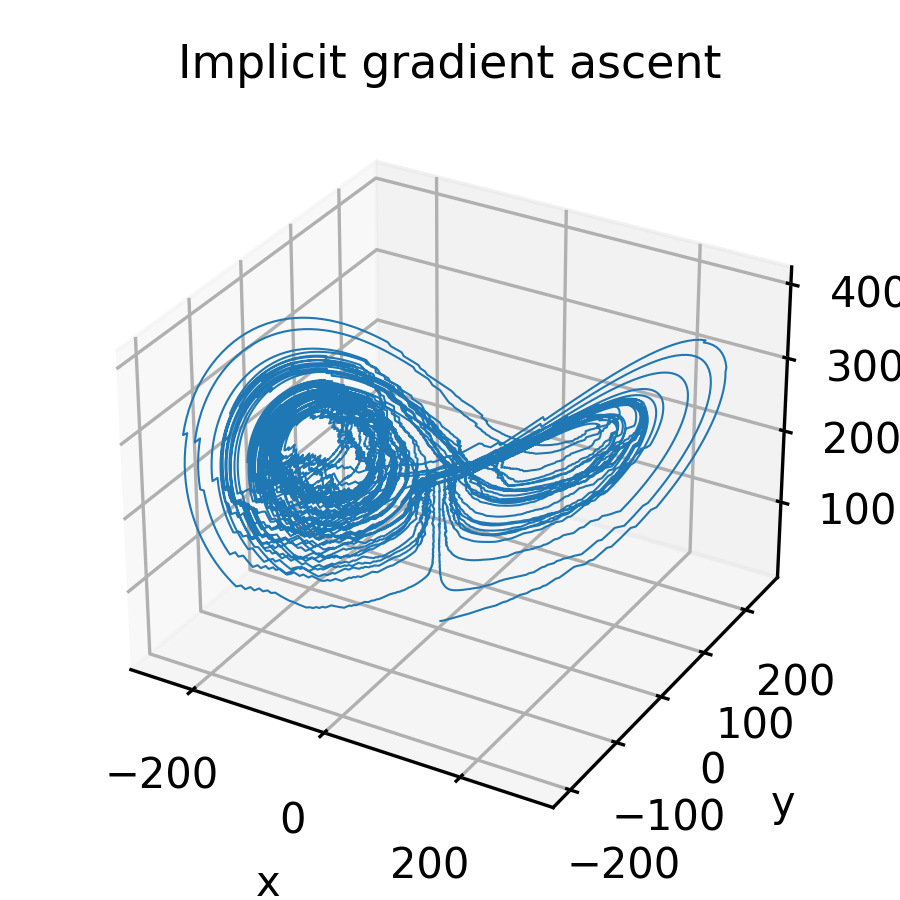}
  \caption{}
  \label{fig:pseudolorenzdynamic}
\end{subfigure}
\begin{subfigure}{.32\textwidth}
  \centering
  \includegraphics[width=1\linewidth]{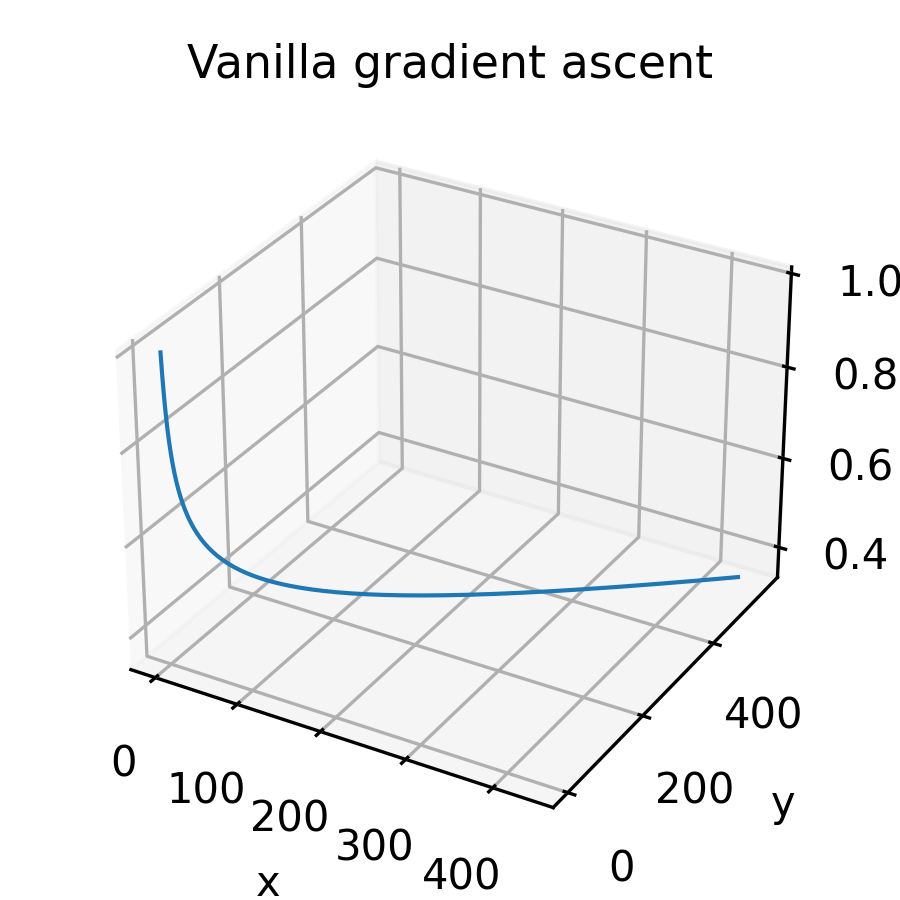}
  \caption{}
  \label{fig:lorenzvanillagradient}
\end{subfigure}
\caption{Implicit gradient ascent (b) outputs a pathological curve with some qualitative aspects of the Lorenz dynamics (a) and really different from a classical gradient (c).}
\end{figure}
The function $g : u \mapsto u^T F(u)$ is a nondegenerate quadratic function whose expression can be found in Appendix~\ref{app:lorenzquadraticform}. The function $g$ has for unique critical point $(0, 0, 0)$ which is a strict saddle-point. We perform gradient ascent with implicit differentiation using \texttt{cvxpylayers} on \eqref{eq:cvxpylayerlorenz}, and the classical gradient ascent on the equivalent problem \eqref{eq:pseudolorenzequivalentproblem}. The path obtained by implicit differentiation (Figure~\ref{fig:pseudolorenzdynamic}) resembles the Lorenz attractor (Figure~\ref{fig:lorenzdynamic}), in stark contrast to the conventional method (Figure~\ref{fig:lorenzvanillagradient}). The chaotic dynamics are a consequence of the lack of invertibility, due to the power $4$ in \eqref{eq:cvxpylayerlorenz}, and various numerical approximations related to optimization and implicit differentiation.

\section{Conclusion and future work}

This article provides a rigorous framework and calculus rules for nonsmooth implicit differentiation using the theory of conservative Jacobians. In particular, it describes precise conditions under which implicit differentiation can be used, in a way that is compatible with backpropagation and first-order algorithms. 

We show the applicability of our results on practical machine learning problems including training of neural networks involving layers with implicitly defined outputs (deep equilibrium nets, networks with optimization layers) and nonsmooth hyperparameter optimization (Lasso-type models).

Finally, we demonstrate the necessity of a rigorous theory of nonsmooth implicit differentiation through multiple numerical experiments. These illustrate the range of extremely pathological gradient dynamics that can occur when algorithmic differentiation is combined with nonsmooth implicit differentiation outside the scope of our theorem, i.e., without satisfying the invertibility condition we specify.

\bibliographystyle{abbrv}
\bibliography{references}
\newpage
\appendix
\etocdepthtag.toc{mtappendix}
\etocsettagdepth{mtsection}{none}
\etocsettagdepth{mtappendix}{section}
\tableofcontents
\section{Lexicon}\label{app:lexicon}

\subsection{Conservative fields}
We first collect the necessary definitions to define a conservative set-valued field, introduced in \cite{bolte2020conservative}, and by extension conservative Jacobians. Recall from multivariable calculus that the \emph{Jacobian} of a differentiable function $f\colon \R^n\to\R^m$ is given by
\newq{
\jac f := \begin{bmatrix}\frac{\partial f_1}{\partial x_1} &\ldots &\frac{\partial f_1}{\partial x_n}\\ \vdots &\ddots &\vdots\\ \frac{\partial f_m}{\partial x_1} &\ldots &\frac{\partial f_m}{\partial x_n}\end{bmatrix}.
}
\begin{definition}[Absolutely continuous curve]
A continuous function $\gamma:\R\to\R^n$ is an absolutely continuous curve if it has a derivative $\dot{\gamma}(t)$, for almost all $t\in\R$, which furthermore satisfies
\begin{equation*}
    \gamma(t) - \gamma(0) = \int\limits_{0}^t \dot{\gamma}(\tau)d\tau
\end{equation*}
for all $t\in\R$.
\end{definition}
The \textit{graph} of a set-valued mapping $D:\R^n\rightrightarrows\R^m$ is the set $\graph D := \{(x,z) : x\in\R^n, z\in D(x)\}$.
\begin{definition}[Closed graph]
A set-valued mapping $D:\R^n\rightrightarrows\R^m$ has closed graph or is graph closed if $\graph D$ is a closed subset of $\R^{n+m}$ or, equivalently, if, for any convergent sequences $(x_k)_{k\in\N}$ and $(z_k)_{k\in\N}$ with $z_k\in D(x_k)$ for all $k\in\N$, it holds
\newq{
\lim\limits_{k\to\infty} z_k\in D\para{\lim\limits_{k\to\infty}x_k}.
}
\end{definition}

\begin{definition}[Locally bounded] A set-valued mapping $D : \R^n \rightrightarrows \R^m$ is locally bounded if for all $x \in \R^n$, there exists a neighborhood $\mcU$ of $x$ and $M > 0$ such that, for all $u \in \mcU$, for all $y \in D(u)$, $\|y\| < M$.
\end{definition}

\begin{definition}[Conservative set-valued field]
A set-valued mapping $D:\R^n\rightrightarrows\R^m$ is a conservative field if the following conditions hold:
\begin{enumerate}
    \item For all $x\in\R^n$, $D(x)$ is nonempty.
    \item $D$ has a closed graph and is locally bounded.
    \item For any absolutely continuous curve $\gamma:[0,1]\to\R^n$ with $\gamma(0)=\gamma(1)$,
    \newq{
    \int\limits_{0}^1\max\limits_{z\in D(\gamma(t))}\langle\dot{\gamma}(t),z\rangle dt=0.
    }
\end{enumerate}
\end{definition}
Although conservative fields are not assumed to be locally bounded in \cite{bolte2020conservative}, we add this restriction here to ensure they are upper semicontinuous. This will allow us to use a nonsmooth Lyapunov method \cite{benaim2005stochastic} to prove convergence of first-order algorithms.
\begin{definition}[Monotone operator]
A set-valued mapping $D:\R^n\rightrightarrows\R^m$ is called a monotone operator if, for all $x,y\in\R^n$, $u\in D(x)$, and $v\in D(y)$,
\newq{
\langle x-y,u-v\rangle\geq 0.
}
\end{definition}

\subsection{A simpler and more operational view on  definability}\label{app:definable}
We recall basic definitions and results on  definable sets and functions used in this work. More details on this theory can be found in \cite{van1996geometric,coste2000introduction}.

\medskip

{\em We make a specific attempt to provide a new simple view on this subject by using dictionaries, in the hope that machine learning users consider utilizing these wonderful tools.}

\medskip

The archetypal o-minimal structure is the collection of \textit{semialgebraic} sets. Recall that a set $A \subset \R^n$ is semialgebraic if it can be written as
\begin{equation*}
    A = \bigcup_{i = 1}^I \bigcap_{j=1}^J\: \{x \in \mathbb{R}^n : P_{ij}(x) < 0, \ Q_{ij}(x) = 0 \}
\end{equation*}
where, for $i \in \{1,..., I\}$ and $j \in \{1,..., J\}$, $P_{ij}$ and  $Q_{ij}$ are polynomials. The  stability properties of semialgebraic sets may be axiomatized \cite{shiota2011geometry, van1996geometric} to give rise to the general notion of an o-minimal structure:

\begin{definition}[o-minimal structure] 
\label{def:ominimal}
Let $\mathcal{O} = (\mathcal{O}_p)_{p \in \mathbb{N}}$ be a collection of sets such that, for all $p \in \NN$, $\mathcal{O}_p$ is a set of subsets of $\R^p$. $\mathcal{O}$ is an o-minimal structure on $(\mathbb{R}, +, \cdot )$ if it satisfies the following axioms:
    \begin{enumerate}
        \item For all $p \in \mathbb{N}$, $\mathcal{O}_p$ is stable by finite intersection and union, complementation, and contains $\R^p$. 
        \item If $A \in \mathcal{O}_p$ then $A \times \mathbb{R}$ and $\mathbb{R} \times A$ belong to $\mathcal{O}_{p+1}$.
        \item\label{def:tarski} Denoting by $\pi$ the projection on the $p$ first coordinates, if $A \in \mathcal{O}_{p+1}$ then $\pi(A) \in \mathcal{O}_p$.
        \item For all $p \in \mathbb{N}$, $\mathcal{O}_p$ contains the algebraic subsets of $\mathbb{R}^p$, i.e., sets of the form $\left\{x \in \mathbb{R}^p : P(x) = 0\right\}$, where $P : \mathbb{R}^p \rightarrow \mathbb{R}$ is a polynomial function. 
        \item The elements of $\mathcal{O}_1$ are exactly the finite unions of intervals.
    \end{enumerate}
\end{definition}

A subset $A \subset \R^n$ is said to be \text{definable} in an o-minimal structure $\mathcal{O} = (\mathcal{O}_p)_{p \in \NN}$ if $\mathcal{O}_n$ contains $A$. A function $f : \R^n \rightarrow \R^m$ is said to be definable if its graph, a subset of $\R^{n + m}$, is definable.

Note that the collection of semialgebraic sets verifies \ref{def:tarski} in Definition~\ref{def:ominimal} according to the Tarski-Seidenberg theorem.

There are several  major structures which have been explored \cite{Wilkie1999ATO,van1996geometric,Dries1995OnTR}. But rather than relying on traditional description of these structures, we provide instead classes of functions that are contained in an o-minimal structure. The goals achieved are twofold:
\begin{itemize}
    \item The classes we provide are o-minimal and thus all the results provided in the main text apply to functions in these classes.
    \item It is very easy to verify that a function belongs to one of the classes. Everything boils down to checking that the problem under consideration can be expressed in one of the dictionaries we provide.
\end{itemize}
Note however that we do not aim at providing neither a comprehensive nor a sharp picture of what could be done with o-minimal structures. 

We consider first a collection of functions which will serve to establish dictionaries:
\begin{enumerate}[label=(\alph*), ref=\alph*]
    \item\label{list:anfun} Analytic functions restricted to semialgebraic compact domains (contained in their natural open domain), examples are $\cos$ and $\sin$ restricted to compact intervals.

    \item\label{list:globsubanfun} ``Globally subanalytic functions'': $\arctan,\tan_{|]-\pi/2,\pi/2[}$ or any functions in (\ref*{list:anfun}) (see \cite{Dries1995OnTR} for a precise definition of global subanalyticity).
    \item\label{list:logexp} The $\log$ and $\exp$ functions.
    \item\label{list:expbase} Functions of the form $x\mapsto x^r$ with $r$ a real constant and $x$ a positive real number. These can be represented as $x\mapsto \exp(r\log(x))$ which is definable in $(\RR,\exp)$.
    \item\label{list:implicit} Implicitly defined semialgebraic functions. That is, functions $G:\Omega\to\R^m$, with $\Omega$ open, which are maximal solutions (i.e., the domain $\Omega$ cannot be chosen to be bigger) to nonlinear equations of the type 
    $$F(x,G(x))=0$$ where $F$ is a semialgebraic function.
\end{enumerate}

With this collection of functions we may build {\em elementary dictionaries}. To demonstrate, we consider the following dictionaries
\begin{align*}
& \dic(\mbox{\ref*{list:anfun}})=\{\mbox{functions satisfying (\ref*{list:anfun})}\}\\
& \dic(\mbox{\ref*{list:expbase}, \ref*{list:implicit})=\{\mbox{functions satisfying  (\ref*{list:expbase}) \text{ or } (\ref*{list:implicit})}}\}\\
& \dic(\mbox{\ref*{list:anfun}, \ref*{list:globsubanfun}, \ref*{list:logexp}, \ref*{list:expbase}, \ref*{list:implicit}})=\{\mbox{functions satisfying (\ref*{list:anfun}) \text{ or }  (\ref*{list:globsubanfun}) \text{ or } (\ref*{list:logexp}) \text{ or }  (\ref*{list:expbase}) \text{ or }  (\ref*{list:implicit})}\}
\end{align*}
The last dictionary describes a larger class of functions, we shall come back on this later on.

Consider the dictionary $\mathcal{D}=\dic(\cdot)$ based on the properties (\ref*{list:anfun})-(\ref*{list:implicit}) described above.

Then, in the spirit of \cite{bolte2020mathematical}, we can extend the idea of piecewise selection functions with the following three definitions.
\begin{definition}[Elementary $\mathcal{D}$-function] An elementary $\mathcal{D}$-function is a $C^2$ function described by a finite compositional expression involving the basic operations $\times, +, /$, multiplication by a constant, and the functions of $\mathcal{D}$ inside their domain of definition. 
\end{definition}
Any elementary $\mathcal{D}$-function is definable in $\R_{\mbox{an,exp}}$ by stability of definable functions by composition. We shall denote $\mathfrak{S}\mathcal{D}$ the set of elementary  $\mathcal{D}$-functions. For instance, the following functions belong to $\mathfrak{S}\mathcal{D}$:

\begin{itemize}
    \item[--] $x \mapsto \frac{1}{1 + \exp(-x)}$. 
    \item[--] $x \mapsto \log(1 + \exp(x))$.
    \item[--] $(\beta, \lambda) \mapsto \|X\beta - Y\|_2 + e^{\lambda} \|\beta\|_1$.
\end{itemize}

\begin{definition}[Elementary $\mathcal{D}$-index]
Consider $r \in \NN^*$,  and $s : \R^n \rightarrow \left\{ 1,\ldots, r\right\}$. Then $s$ is said to be an elementary $\mathcal{D}$-index if, for $i \in \{1, \ldots, r\}$, each of the pre-images $s^{-1}(i)$ (i.e., the points in $\R^n$ such that $s$ selects the index $i$) can be written as
\begin{equation*}
    \bigcup_{i = 1}^I \bigcap_{j=1}^J\: \{x \in \mathbb{R}^n : g_{ij}(x) < 0, \ h_{ij}(x) = 0 \}
\end{equation*}
where, for $i \in \{1, \ldots, I\}$ and $j \in \{1, \ldots, J\}$, the $g_{ij}$ and $h_{ij}$ are elementary $\mathcal{D}$-functions.
\end{definition}

\begin{definition}[Piecewise $\mathcal{D}$-function] A function $f : \R^n \rightarrow \R^m$ is a piecewise $\mathcal{D}$-function if there exist $r \in \NN^*$, elementary $\mathcal{D}$-functions $f_1,\ldots, f_r$, and an elementary $\mathcal{D}$-index $s : \R^n \rightarrow \left\{1, \ldots, r\right\}$ such that for all $x \in \R^n$,
\begin{equation*}
   f(x) = f_{s(x)} (x).
\end{equation*}
\end{definition}

We denote $\mathcal{P} \mathcal{D}$ the set of piecewise $\mathcal{D}$-functions.
With the assumptions we have on the dictionary $\mathcal D$, the piecewise selections we consider are all definable (it 's not always the case in general).
Notice that piecewise log-exp functions \cite{bolte2020mathematical} are a specific case of $\mathcal{D}$-functions with the dictionary $\mathcal{D} = \dic(\mbox{\ref*{list:logexp}})= \left\{\log, \exp\right\}$. It is easy to see that the following functions are in $\mathcal{P}\mathcal{D}$ and thus definable:

\begin{itemize}
    \item[--] $x \mapsto \max(0, x)$ (relu).
    \item[--] $x \mapsto \max(x_1, ..., x_n)$.
    \item[--] sort function.
    \item[--] $x \mapsto  \begin{cases}
 \frac{1}{2}{x^2}                   & \text{for } |x| \le \delta, \\
 \delta (|x| - \frac{1}{2}\delta), & \text{otherwise,}
\end{cases}$ with $\delta > 0$ (Huber loss).
\end{itemize}

Moreover, composition of functions from $\mathcal{P} \mathcal{D}$ are definable. This allows to say that if $\rho(w, x)$ is the output of a neural network built with usual elementary blocks (for instance Dense, Max Pooling or Conv layers), or even implicit layers involving functions in $\mathcal{P}\mathcal{D}$,  with input $x$ and weights $w$, then the empirical risk $ \frac{1}{N} \sum_{i = 1}^N \ell(\rho(w,x_i), y_i)$ is definable with respect to $w$ provided that $\ell$ is also in $\mathcal{P} \mathcal{D}$.

\begin{remark}{\rm
(a) (Small and big dictionaries) It may be puzzling for the reader to see that there is a dictionary that contains all the others. A major comment is in order: bigger is not always better. The bigger the dictionary is, the weaker some properties are. For instance, any piecewise selection $f:\R^n\to \R$ built upon $\dic(\mbox{\ref*{list:anfun}, \ref*{list:globsubanfun}})$ satisfies $\|f(x)\|\leq c \|x\|^N$ for some $c>0,N>0$, which may have consequences in terms of convergence rates, see e.g., \cite{attouch2010proximal}. Thus in practice using the smallest dictionary possible may lead to sharper results. On top of this, there are no universal dictionaries \cite{Dries1995OnTR}.\\
(b) (PAP functions and definability) Recently PAP functions were introduced in order to deal with automatic differentiation matters \cite{lee2020correctness}. To deal with such types of functions in our framework and have guarantees in terms of automatic differentiation, implicit differentiation or convergence properties, we need to view them through the dictionary paradigm. For this we consider the dictionary of analytic functions defined on $\R^p$ for some $p$. In that case, piecewise functions are not necessarily definable but their restrictions to any ball (or any compact semialgebraic subset) are definable.}
\end{remark}

\section{Results from Section~\ref{sec:conservativeJacobians}}\label{app:conservativeJacobians}

\begin{apptheorem}{\ref{th:incluseionClarke}}[The Clarke Jacobian is a minimal conservative Jacobian]
	Given a nonempty open subset $\mathcal{U}$ of $\R^n$ and $F : \mathcal{U}\subset \R^n \to \R^m$ locally Lipschitz, let $J_F$ be a convex-valued conservative Jacobian for $F$. Then for almost all $x \in \mathcal{U}, J_F(x) = \left\{\jac F \right\}$ and  for all $x \in \mathcal{U}$, $\jac^c F(x) \subset J_F(x)$. 
\end{apptheorem}
\begin{proof}
Using \cite[Lemma 4]{bolte2020conservative} for $i \in \{1, \ldots, m\}, \:\left[J_F\right]_i$ is a conservative map for $F_i$ on $\mathcal{U}$ and it is equal to $\nabla F_i$ on a set of full measure $S_i\subset \mathcal{U}$. Hence for all $x \in S :=  \bigcap_{i = 1}^m S_i$, which is of full measure in $\mathcal{U}$,  $J_F(x)=\jac F(x)$. Since $S$ has full measure within $\mathcal U$, \cite{Warga1981FatHA} gives  the representation
\begin{align*}
    \jac^c F(x) = \operatorname{conv} \left\{ \lim_{k \to +\infty} \jac F(x_k) : x_k \in S, x_k \underset{k \to +\infty}{\xrightarrow[]{}} x\right\}, \mbox{ for any } x \in \mathcal{U}.
\end{align*}
But since $J_F$ coincides with $\jac F$ throughout $S$, we have
\begin{align*}
    \jac^c F(x) = \operatorname{conv} \left\{ \lim_{k \to +\infty} J_F(x_k) : x_k \in  S, x_k \underset{k \to +\infty}{\xrightarrow[]{}} x\right\}
\end{align*}
for each $x\in\mathcal{U}$. Finally, by graph closedness and convexity of $J_F$ we get, for each $x\in\mathcal{U}$,	
\begin{align*}
    \jac^c F(x) & \subset \operatorname{conv} \left\{  J_F\left(\lim_{k \to +\infty} x_k\right) : x_k \in S, x_k \underset{k \to +\infty}{\xrightarrow[]{}} x\right\} = J_F(x).
\end{align*}
\end{proof}

\begin{appprop}{\ref{prop:conservfieldsdecomp}}[Decomposition of conservative fields]\label{app:conservfieldsdecomp}
Let $J_F$ be a conservative Jacobian for $F$, then there is a residual $R$ such that
\begin{align*}
    J_F\subset \jac^c F +R.
\end{align*}
\end{appprop}

\begin{proof}
We have obviously the inclusion 
$$J_F\subset \jac^c F +(J_F-\jac^c F),$$ 
so it suffices to remark that $(J_F-\jac^c F)$ is residual due to the conservativity properties of both $J_F$ and $\jac^c F$.
\end{proof}

\begin{apptheorem}{\ref{th:implicitMainTheorem}}[Implicit differentiation]
Let $F:\R^n\times\R^m\to \R^m$ be path differentiable on $\mcU\times\mcV\subset\R^n\times \R^m$ an open set and $G:\mcU\to\mcV$ a locally Lipschitz function such that, for each $x\in \mcU$,
\nnewq{\label{eq:Gsolvesapp}
F(x,G(x))=0.
}
Furthermore, assume that for each $x\in \mcU$, for each $[A\ B]\in \cmap_F(x,G(x))$, the matrix $B$ is invertible where $\cmap_F$ is a conservative Jacobian for $F$. Then, $G:\mcU\to\mcV$ is path differentiable with conservative Jacobian given, for each $x\in\mcU$, by
\newq{
\cmap_{G} \colon x\rightrightarrows \left\{-B^{-1}A : [A\ B]\in \cmap_F(x, G(x))\right\}.
}
\end{apptheorem}
\begin{proof}
Let $\gamma:[0,1]\to \mcU$ be absolutely continuous, then the composition $G\circ \gamma$ is also absolutely continuous since $G$ is locally Lipschitz. By \eqref{eq:Gsolvesapp} we have, for all $t\in[0,1]$,
\newq{
F(\gamma(t),G(t))) = 0
}
which we can differentiate almost everywhere; for almost every $t\in[0,1]$, for any $[A\ B]\in \cmap_F(\gamma(t),G(\gamma(t)))$,
\newq{
[A\ B]\begin{bmatrix}\dgamma(t)\\ \frac{\mathrm{d}}{\mathrm{d} t}G(\gamma(t))\end{bmatrix}=0\implies -A\dgamma(t)=B\frac{\mathrm{d}}{\mathrm{d}t}G(\gamma(t)).
}
Since $B$ is assumed to be invertible, we have, for almost every $t\in[0,1]$,
\newq{
-B^{-1}A\dgamma(t) = \frac{\mathrm{d}}{\mathrm{d}t} G(\gamma(t)).
}
The set-valued mapping $\cmap_G \colon x\rightrightarrows  \left\{-B^{-1}A: [A\ B]\in \cmap_F(x,G(x))\right\}$ is nonempty, locally bounded, and has a closed graph for each $x\in \mcU$ since $\cmap_F(x,G(x))$ is a conservative Jacobian and $B$ is invertible . We conclude that $G$ is path differentiable on $\mcU$ with conservative Jacobian $\cmap_G$.
\end{proof}

\begin{appcor}{\ref{cor:implicitFun}}[Path differentiable implicit function theorem]
Let $F:\R^n\times\R^m\to \R^m$ be path differentiable with conservative Jacobian $J_F$. 
Let $(\hat{x},\hat{y})\in\R^n\times \R^m$ be such that $F(\hat{x},\hat{y})=0$. Assume that $\cmap_F(\hat{x},\hat{y})$ is convex and that, for each $[A\ B]\in \cmap_F(\hat{x},\hat{y})$, the matrix $B$ is invertible. Then, there exists an open neighborhood $\mcU\times\mcV\subset\R^n\times\R^m$ of $\para{\hat{x},\hat{y}}$ and a path differentiable function $G:\mcU\to\mcV$ such that the conclusion of Theorem~\ref{th:implicitMainTheorem} holds.
\end{appcor}

\begin{proof}
Since $\cmap_F(\hat{x},\hat{y})$ is convex, it follows from Theorem~\ref{th:incluseionClarke} that $\cjac F(\hat{x},\hat{y})\subset \cmap_F(\hat{x},\hat{y})$ and thus, for any $[A\ B]\in \cjac F(\hat{x},\hat{y})$, $B$ is invertible, i.e., the conditions to apply \cite[7.1 Corollary]{clarke1983optimization} to $F$ are satisfied. Therefore there exists an open neighborhood $\mcU_1\times\mcV_1\subset\R^n\times\R^m$ of $\para{\hat{x},\hat{y}}$ and a locally Lipschitz function $G:\mcU_1\to\mcV_1$ such that, for all $x\in\mcU_1$,
\newq{
F(x,G(x)) = 0.
}
By the continuity of the determinant and the fact that $\cmap_F$ has a closed graph, there exists an open neighborhood $\mcU_2\times\mcV_2\subset\R^n\times\R^m$ of $(\hat{x},\hat{y})$ such that, for all $(x,y)\in \mcU_2\times\mcV_2$, for all $[A\ B]\in J_F(x, y)$, the matrix $B$ is invertible. Let $\mcU\times\mcV:= (\mcU_1\cap\mcU_2)\times (\mcV_1\cap\mcV_2)$, which is an open neighborhood of $\para{\hat{x},\hat{y}}$. Then the requirements of Theorem~\ref{th:implicitMainTheorem} are met for $F$, $\cmap_F$, and $G$ on $\mcU\times\mcV$ and the desired claims follow.
\end{proof}

\begin{appcor}{\ref{cor:inverse function}}[Path differentiable inverse function theorem]
Let $\mcU$ and $\mcV$ be open neighborhoods of $0$ in $\R^n$ and $\Phi:\mcU \to \mcV$ path differentiable  with $\Phi(0)=0$. Assume that $\Phi$ has a conservative Jacobian $\cmap_\Phi$ such that $\cmap_\Phi(0)$ contains only invertible matrices. Then, locally, $\Phi$ has a path differentiable inverse $\Psi$ with a conservative Jacobian given by
$$\cmap_\Psi(y)=\left\{A^{-1}: A\in J_{\Phi}(\Psi(y))\right\}.$$
\end{appcor}
\begin{proof} Consider the function $F(x,y)=x-\Phi(y)$ and observe that it satisfies the assumptions of Corollary \ref{cor:implicitFun}, so that we obtain a function $G$ which is exactly the desired inverse.
\end{proof}

\medskip
 It is tempting to think that Corollary~\ref{cor:inverse function} should come with a formula of the type
$$\cjac\Psi(z)=[\cjac \Phi(\Psi(z))]^{-1},$$
for all $z$ in a neighborhood of $0$. This happens to be false, making the use of the notion of conservativity necessary to catpure the artifacts resulting from application of ordinary calculus rules to nonsmooth inverse functions. Note that since the inverse function theorem is a special case of the implicit function theorem, this also rules out a Clarke calculus for implicit functions. 

\begin{example}[\bf Counterexample to a potential  ``Clarke implicit differential calculus'']
\label{ex:counterExClarkeInverse}
\normalfont
We follow the example given by Clarke \cite[Remark 7.1.2]{clarke1983optimization}. Consider the  mapping $\Phi:\R^2\to \R^2$ given by 
$$\Phi(x,y)=\left(|x|+y,2x+|y|\right).$$
It is locally Lipschitz and semialgebraic and thus path differentiable with its Clarke Jacobian a conservative Jacobian. We have the following explicit piecewise linear representation
\begin{align*}
    \Phi(x,y)= 
    \begin{cases}
        (x + y, 2x + y) \qquad &\text{ if }\ x \geq 0 \mbox{ and } y \geq 0,\\
        (x + y, 2x - y) \qquad &\text{ if }\ x \geq 0 \mbox{ and } y \leq 0,\\
        (-x + y, 2x - y) \qquad &\text{ if }\ x \leq 0 \mbox{ and } y \leq 0, \\
        (-x + y, 2x + y) \qquad &\text{ if }\ x \leq 0 \mbox{ and } y \geq 0
    \end{cases}
\end{align*}
from which we deduce that the Clarke Jacobian of $\Phi$ has the following structure
\begin{align*}
& \cjac \Phi(0)=\conv \left\{\begin{bmatrix}
    1 & 1\\
    2 & 1\\
\end{bmatrix},
\begin{bmatrix}
    1 & 1\\
    2 & -1\\
\end{bmatrix},
\begin{bmatrix}
    -1 & 1\\
    2 & -1\\
\end{bmatrix},
\begin{bmatrix}
    -1 & 1\\
    2 & 1\\
\end{bmatrix}\right\}
\end{align*}
where the matrices correspond to linear maps in the explicit definition of $\Phi$. Therefore $\cjac \Phi(0)$ is an affine set whose dimension is $2$. In addition, it contains only invertible matrices \cite[Remark 7.1.2]{clarke1983optimization}. We will use the following explicit matrix inverses:
\begin{align*}
\begin{bmatrix}
    1 & 1\\
    2 & 1\\
\end{bmatrix}^{-1}=
\begin{bmatrix}
    -1 & 1\\
    2 & -1\\
\end{bmatrix}, \quad
\begin{bmatrix}
    1 & 1\\
    2 & -1\\
\end{bmatrix}^{-1} = \frac{1}{3} 
\begin{bmatrix}
    1 & 1\\
    2 & -1\\
\end{bmatrix}, \quad
\begin{bmatrix}
    -1 & 1\\
    2 & 1\\
\end{bmatrix}^{-1} = \frac{1}{3}
\begin{bmatrix}
    -1 & 1\\
    2 & 1\\
\end{bmatrix}.
\end{align*}

\begin{figure}[t]
    \centering
    \includegraphics[width=0.5\textwidth]{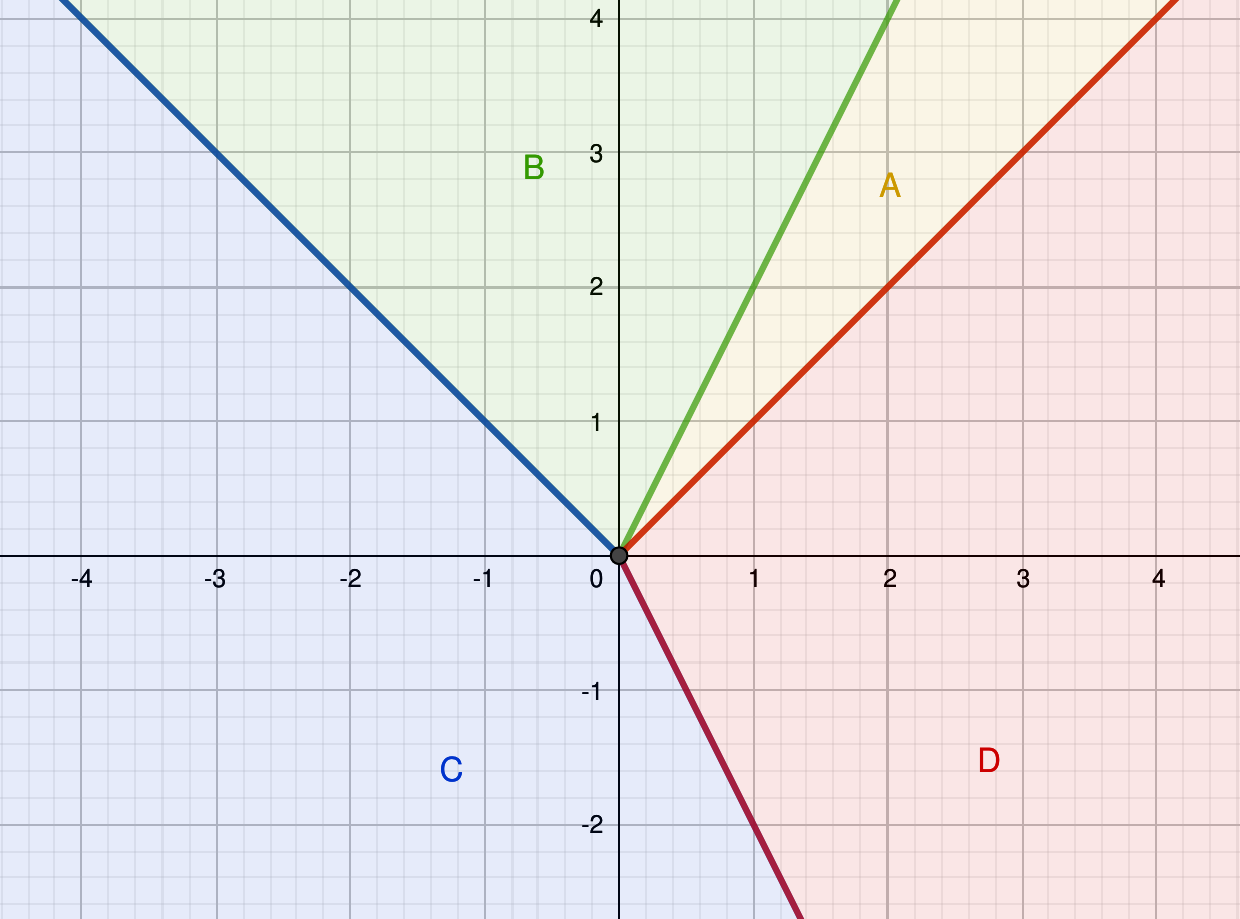}
    \caption{\footnotesize Illustration of the four different sets in the explicit piecewise affine representation of $\Psi = \Phi^{-1}$.}
    \label{fig:illustrCounterEx}
\end{figure}

Using the above, one can verify that $\Phi$ is a homeomorphism whose inverse is also piecewise linear. We set $\Psi = \Phi^{-1}$; it is  given by 
\begin{alignat*}{2}
&\Psi(u,v)=\left(v-u,2u-v\right) &&\quad\mbox{for }(u,v)\in A,\\
&\Psi(u,v)=\frac{1}{3}\left(u+v,2u-v\right) &&\quad\mbox{for }(u,v)\in B,\\
&\Psi(u,v)=(u+v,2u+v) &&\quad\mbox{for }(u,v)\in C,\\
&\Psi(u,v)=\frac{1}{3}  \left(v-u,2u+v\right) &&\quad\mbox{for }(u,v)\in D,
\end{alignat*}
where the subsets $A,B,C,D$ form a ``partition''\footnote{Each piece having two half lines in common with other pieces.} of $\R^2$
\begin{align*}
& A=\left\{(u,v)\in\R^2: v-u\geq0, 2u-v\geq 0\right\}&(\text{corresponding to } x \geq 0, y \geq 0),\\
& B=\left\{(u,v)\in\R^2: u+v\geq0, 2u-v\leq 0\right\}&(\text{corresponding to } x \geq 0, y \leq 0),\\
& C=\left\{(u,v)\in\R^2: u+v\leq0, 2u+v\leq 0\right\}&(\text{corresponding to } x \leq 0, y \leq 0),\\
& D=\left\{(u,v)\in\R^2: v-u\leq0, 2u+v\geq 0\right\}&(\text{corresponding to } x \leq 0, y \geq 0).
\end{align*}
A graphical representation of these sets is given in Figure \ref{fig:illustrCounterEx}.

From this explicit piecewise linear representation of $\Psi$, we deduce that its Clarke Jacobian at $0$ is the following
\begin{align*}
& \cjac \Psi(0)=\conv \left\{\begin{bmatrix}
    1 & 1\\
    2 & 1\\
\end{bmatrix},
\begin{bmatrix}
    -1 & 1\\
    2 & -1\\
\end{bmatrix},
\frac{1}{3}
\begin{bmatrix}
    -1 & 1\\
    2 & 1\\
\end{bmatrix},
\frac{1}{3}
\begin{bmatrix}
    1 & 1\\
    2 & -1\\
\end{bmatrix}\right\}.
\end{align*}
For a given subset of linear space we denote by $\aff\, F$ the affine span of $F$. It is easy to see that 
$\dim \aff [\,\cjac \Phi(0)]=2 $ while $\dim \aff\, [\cjac \Psi(0)]=3$. More concretely, vectorialize the set $\cjac \Psi(0)$ at $M=\frac{1}{3} \begin{bmatrix}
    1 & 1\\
    2 & -1\\
\end{bmatrix}$ by considering the matrices given by 
\begin{align*}
\begin{bmatrix}
    1 & 1\\
    2 & 1\\
\end{bmatrix}-M,\quad
\begin{bmatrix}
    -1 & 1\\
    2 & -1\\
\end{bmatrix}-M,\quad
\frac{1}{3}
\begin{bmatrix}
    -1 & 1\\
    2 & 1\\
\end{bmatrix}-M
\end{align*}
that is 
\begin{align*}
\frac{1}{3}
\begin{bmatrix}
    2 & 2\\
    4 & 4\\
\end{bmatrix}
,\quad
\frac{1}{3}
\begin{bmatrix}
    -4 & 2\\
    4 & -2\\
\end{bmatrix},
\quad
\frac{1}{3}
\begin{bmatrix}
    -2 & 0\\
    0 & 2\\
\end{bmatrix}.
\end{align*}
These matrices are independent so that $\cjac \Psi(0)$ is an affine set whose dimension is $3$.

Matrix inversion is a semialgebraic diffeomorphism (when restricted to invertible matrices) so it preserves dimension. For this reason the set $[\cjac \Psi(0)]^{-1} = \{M^{-1}, M \in \cjac \Psi(0)\}$ is a semialgebraic set of dimension $3$, and we have 
\begin{align}
   [\cjac \Psi(0)]^{-1}\not\subset[\cjac \Phi(0)].\label{pasinc}
\end{align}
However, we have shown that   $z\mapsto [\cjac \Psi(\Phi(z))]^{-1}$ is a conservative Jacobian. This example excludes the possibility of a simple inverse (implicit) function theorem with a ``Clarke Jacobian calculus'' and illustrates the requirement for a more flexible notion (conservativity) when using calculus rules in an implicit function (or inverse function) context.

\end{example}

\paragraph{The Lipschitz definable implicit and inverse function theorems.}

In the definable (e.g. semialgebraic case) our results have a remarkably simple expression that we give below.
\begin{theorem}[Lipschitz definable inverse function theorem]\label{th:invdef}
Let $\mcU$ and $\mcV$ be two open neighborhoods of $0$ in $\R^n$ and $\Phi:\mcU \to \mcV$ a locally Lipschitz definable mapping with $\Phi(0)=0$. Assume that $\Phi$ has a conservative Jacobian $J_\Phi$ such that $J_\Phi(0)$ contains only invertible matrices. Then, locally, $\Phi$ has locally Lipschitz definable inverse $\Psi$ with a conservative Jacobian given by
$$J_\Psi(y)=\left\{A^{-1}: A\in J_{\Phi}(\Psi(y))\right\}.$$
\end{theorem}
\begin{proof}
It suffices to use the fact that definable mappings are path differentiable, see \cite{bolte2020conservative}, and that the the graph of $\Psi$ is given by a first-order formula.
\end{proof}

The same type of arguments gives:

\begin{theorem}[Lipschitz definable implicit function theorem]\label{th:inpdef}
Let $F:\R^n\times\R^m\to \R^m$ be locally Lipchitz and definable with conservative Jacobian $J_F$. 
Let $(\hat{x},\hat{y})\in\R^n\times \R^m$ be such that $F(\hat{x},\hat{y})=0$. Assume that $\cmap_F(\hat{x},\hat{y})$ is convex and that, for each $[A\ B]\in \cmap_F(\hat{x},\hat{y})$, the matrix $B$ is invertible. Then, there exists an open neighborhood $\mcU\times\mcV\subset\R^n\times\R^m$ of $\para{\hat{x},\hat{y}}$ and a locally Lipschitz definable function $G:\mcU\to\mcV$ such that, for all $x\in\mcU$,
\newq{
F(x,G(x)) = 0.
}
Moreover, for each $x\in\mcU$, the mapping
$
\cmap_{G}\colon x \rightrightarrows \left\{-B^{-1}A : [A\ B]\in \cmap_F(x, G(x))\right\}
$
is conservative for $G$.

\end{theorem}

\section{Results from Section~\ref{sec:examplesML}}\label{app:examplesML}

\subsection{Monotone operator deep equilibrium networks}
\label{app:MONDEQ}
\begin{appprop}{\ref{prop:DEQ}}[Path differentiation through monotone layers]
    Assume that $\cmap_\sigma$ is convex-valued and that, for all $J \in \cmap_\sigma(Wz(W,b)+b)$, the matrix $(\Id_m - JW)$ is invertible.  
Consider a loss-like function $\ell \colon \RR^m \to \RR$ with conservative gradient $D_\ell \colon \RR^m \rightrightarrows \RR^m$, then  $g:(W,z) \mapsto \ell(z(W,b))$ is path differentiable and has a conservative gradient $D_g$ defined through
    \begin{align*}
        D_g \colon (W,b) \rightrightarrows \left\{ J^T(\Id_m - JW)^{-T}v z^T,  J^T(\Id_m - JW)^{-T}v): \,\, J \in J_\sigma(Wz + b),\, v \in D_\ell(z)  \right\}.
    \end{align*}
\end{appprop}

\begin{proof}
    The quantity $z(W,b)$ is defined implicitly by the relation
    \begin{align}
        z(W,b) - \sigma(W z(W,b) + b) = 0.
        \label{eq:implicitz}
    \end{align}
    
    We set $M =m +  m + m \times m$ and represent the pair $(W,b) \in \RR^{m \times m} \times \RR^m$ as $(w_1,\ldots,w_m,b) \in \RR^{M - m}$ where $w_i \in \RR^m$ is the $i$-th row of $W$ for $i \in \{1,\ldots,m\}$. We denote by $\mathcal{B}\colon \RR^M \to \RR^m$ the bilinear map defined as
    \begin{align*}
        \mathcal{B}(w_1,\ldots, w_m, b, z) := Wz + b
    \end{align*}
    so that $\mathcal{B}$ is infinitely differentiable. Equation~\eqref{eq:implicitz} is then equivalent to 
    \begin{equation*}
        z - (\sigma \circ \mathcal{B})(w_1,\ldots, w_m, b, z) = 0.
    \end{equation*}
    We denote by $F$ the mapping 
    \begin{align*}
        F \colon (w_1,\ldots, w_m, b, z) \mapsto z - (\sigma \circ \mathcal{B})(w_1,\ldots, w_m, b, z).
    \end{align*}
     For $i \in \{1 \ldots m\}$, denote by $Z_i \in \RR^{m \times m}$ the matrix whose $i$-th row is $z$, and remaining rows are null. The Jacobian of $\mathcal{B}$, $\jac \mathcal{B} \colon \RR^M \to \RR^{m \times M}$ is  as follows:
    \begin{align*}
        \jac \mathcal{B}(w_1,\ldots, w_m, b, z) = \left[Z_1\quad \ldots \quad Z_m \quad\Id_m \quad W \right]
    \end{align*}
    where $[A\,B]$ is used to denote the columnwise concatenation of matrices $A$ and $B$. By hypothesis, we have a conservative Jacobian for $\sigma$, $J_\sigma$. Conservative Jacobians may be composed as usual Jacobians \cite[Lemma 5]{bolte2020conservative}. As $\mathcal{B}$ is continuously differentiable, $\jac \mathcal{B}$ is also a conservative Jacobian for $\mathcal{B}$. Therefore, we have the following conservative Jacobian for $F$, 
    \begin{align*}
        J_F (w_1,\ldots, w_m, b, z) \rightrightarrows \left\{ \left[-J Z_1 \quad\ldots\quad -J Z_m\quad -J\quad \Id_m-J W \right],\, J \in J_\sigma(Wz + b)  \right\}.
    \end{align*}
    Finally, by hypothesis, for any $W,b$, and $z$ such that $F(W,b,z) = 0$ and any $J\in J_\sigma(Wz + b)$, the matrix $\Id_m-J W$ is invertible. Therefore, Theorem \ref{th:implicitMainTheorem} applies and, setting $\tilde{M} = m\times m + m = M - m$, the set-valued mapping
    \begin{align*}
        J_z \colon \RR^{\tilde{M}} &\rightrightarrows \RR^{m \times \tilde{M}}\\
        (w_1,\ldots, w_m, b) &\rightrightarrows \left\{ (\Id_m-J W)^{-1} J \left[Z_1 \quad\ldots\quad Z_m\quad \Id_m \right],\, J \in J_\sigma(Wz + b)  \right\}
    \end{align*}
    is conservative for $(W,b) \mapsto z(W,b)$ as defined in \eqref{eq:implicitz}. We denote by $Z \in \RR^{m \times \tilde{M}}$ the matrix $\left[Z_1 \ldots Z_m \: \Id_m \right]$ appearing in the definition of $J_z$. Given the loss function $\ell$, the mapping $J_\ell \colon z \mapsto \{v^T,\,v \in D_\ell(z)\}$ is a conservative Jacobian for $\ell$ \cite[Lemma 3]{bolte2020conservative} and therefore, the set-valued mapping
    \begin{align*}
        J_g \colon \RR^{\tilde{M}} &\rightrightarrows \RR^{1 \times \tilde{M}}\\
        (w_1,\ldots, w_m, b) &\rightrightarrows \left\{ v^T (\Id_m-J W)^{-1} J Z,\, J \in J_\sigma(Wz + b), \, v \in D_\ell(z(W,b))  \right\}
    \end{align*}
    is a conservative Jacobian for $g \colon (W,b) \mapsto \ell(z(W,b))$. Using \cite[Lemma 4]{bolte2020conservative}, we obtain a conservative gradient field for $g$ by a simple transposition as follows
    \begin{align*}
        D_g \colon (w_1,\ldots, w_m, b) \rightrightarrows \left\{ Z^T J^T  (\Id_m-J W)^{-T} v,\, J \in J_\sigma(Wz + b), \, v \in D_\ell(z(W,b))  \right\}.
    \end{align*}
    We now identify the terms by block computation; recall that $Z = \left[Z_1\, \ldots\, Z_m\, \Id_m \right]$ and that $Z_i \in \RR^{m \times m}$ is the matrix whose $i$-th row is $z$ with remaining rows null for each $i\in \{1, \ldots, m\}$. The term associated to $b$ corresponds to the last $m \times m$ block in $Z$, it is indeed of the form $J^T  (\Id_m-J W)^{-T} v$. Similarly, for each $i \in\{1,\ldots, m\}$, the term associated to $w_i$ is of the form $Z_i^T J^T  (\Id_m-J W)^{-T} v$. For any $a \in \RR^m$ and $i \in \{1,\ldots,m\}$, we have $Z_i^T a = a_i z$ where $a_i$ is the $i$-th coordinate of $a$ and $z$ corresponds to the $i$-th row of $Z_i^T$. So the component associated to $w_i$ in $D_g$ is of the form $[ J^T  (\Id_m-J W)^{-T} v]_i z$, where $[\cdot]_i$ denotes the $i$-th coordinate. Since $w_i$ denotes the $i$-th row of $W$, rearranging this expression in matrix format provides a term of the form $J^T (\Id_m-J W)^{-T} v z^T$ for the $W$ component. This concludes the proof.
\end{proof}

\subsection{Optimization layers: the conic program case}
\label{app:cone}
Let us first expand on the link between zeros of the residual map and KKT solutions. We provide a simplified view of \cite{busseti2019solution,agrawal2019differentiating}, ignoring cases of infeasibility and unboundedness. Note that this corresponds to enforcing $w = 1$ as done in \cite{Agrawal2019differentiable,agrawal2019differentiating}.

The following is due to Moreau \cite{moreau1962decomposition}. Recall that the polar of a closed convex cone $\mathcal{K} \subset \RR^m$ is given by $\mathcal{K}^\circ = \left\{x \in \RR^m,\, y^T x \leq 0,\, \forall y \in \mathcal{K} \right\}$, in which case $(\mathcal{K}^\circ)^\circ = \mathcal{K}$ and the dual cone satisfies $\mathcal{K}^* = - \mathcal{K}^\circ$.
\begin{proposition}
    Let $s,y,v\in\RR^m$; the following are equivalent
    \begin{itemize}
        \item $v = s + y$, $s \in \mathcal{K}$, $y \in \mathcal{K}^\circ$, $s^Ty = 0$.
        \item $s = P_{\mathcal{K}} (v)$, $y = P_{\mathcal{K}^\circ} (v)$.
    \end{itemize}
    \label{prop:moreauCone}
\end{proposition}
We may reformulate this equivalence as follows, using changes of signs on $y$ and $v$, noticing that $- P_{\mathcal{K}^\circ} (- \cdot) = P_{\mathcal{K}^*}(\cdot)$ since $\mathcal{K}^* = - \mathcal{K}^\circ$,
\begin{itemize}
    \item[(i)] $v = y-s$,\quad $s \in \mathcal{K}$,\quad $y \in \mathcal{K}^*$,\quad $s^Ty = 0$.
    \item[(ii)] $s = P_{\mathcal{K}^*} (v) - v$,\quad $y = P_{\mathcal{K}^*} (v)$.
\end{itemize}
Now the KKT system in $(x,y,s)$ for problem (P) and (D) can be written as follows (see, for example, \cite{busseti2019solution}),
\begin{alignat*}{2}
    A^T y + c &= 0, &&\quad y \in \mathcal{K}^* \\
    -Ax + b &= s, &&\quad s \in \mathcal{K} \\
    s^T y &= 0 && 
\end{alignat*}
which is equivalent, by setting $v = y-s$ and $u = x$, to 
\nnewq{\label{eq:lastsys}
    A^T P_{\mathcal{K}^*} (v)  + c &= 0 \\
    -Au + b &= P_{\mathcal{K}^*} (v) - v
}
The system \eqref{eq:lastsys} is equivalent to $\mathcal{N}(z,A,b,c) = 0$ with $z = (u,v)$. We have shown that $(x,y,s)$ is a KKT solution to the system if and only if $(x,y,s) = (u, P_{\mathcal{K}^*} (v), P_{\mathcal{K}^*} (v) - v) = \phi(z)$ for $z = (x, y-s)$ such that $\mathcal{N}(z,A,b,c) = 0$.

\begin{appprop}{\ref{prop:pathdiffconeprog}}[Path differentiation through cone programming layers]
Assume that $P_\mathcal{K^*}$, $\mathcal{N}$ are path differentiable,  denote respectively by $\cmap_{P_\mathcal{K^*}}$, $\cmap_\mathcal{N}$ corresponding convex-valued conservative Jacobians. Assume that for all $A,b,c \in \RR^{m\times n}\times \RR^m\times \RR^n$, $z = \nu(A,b,c) \in \RR^n \times \RR^m$ is the unique solution to $\mathcal{N}(z,A,b,c) = 0$, and that all matrices formed from the $N$ first  columns of $\cmap_\mathcal{N}(z, A, b, c)$ are invertible. 
Then, $\phi$, $\nu$, and $\sol$ are path differentiable functions with conservative Jacobians:
    \begin{align*}
&     \cmap_{\nu}(A,b,c) :=  \left\{-U^{-1}V: [U\ V]\in \cmap_{\mathcal{N}}(\nu(A,b,c),A,b,c)\right\},\\
& \cmap_\phi(z) := \begin{bmatrix}
        \Id_n & 0  \\
        0&\cmap_{P_{\mathcal{K}^*}}(v)  \\
        0 &(\cmap_{P_{\mathcal{K}^*}}(v)- \Id_m)  
    \end{bmatrix},\\
& \cmap_{\sol}(A,b,c) :=  \cmap_\phi(\nu(A,b,c)) \cmap_\nu(A,b,c).
\end{align*}
\end{appprop}
\begin{proof} 
First, the assumptions clearly ensure that $\nu$ and $\sol$ are single-valued and can be interpreted as functions such that $\sol = \phi \circ \nu$. By assumption, $\phi$ is differentiable. We will first use Corollary \ref{cor:implicitFun} to obtain a conservative Jacobian for $\nu$ and then justify the expression for $\phi$. The composition obtained for $J_\sol$ results from Proposition~\ref{prop:jacobianComposition}.

Let $A,b,c \in \RR^{m\times n}\times \RR^m\times \RR^n$, $z := (u,v) \in \RR^n \times \RR^m $ such that $\mathcal{N}(z,A,b,c) = 0$. By assumption, the submatrices formed from the first $N$ columns of $\cmap_\mathcal{N}(z,A,b,c)$ are invertible. Then applying Corollary~\ref{cor:implicitFun}, there exist open neighborhoods $\mcU \subset \R^{m \times n} \times \R^{m} \times \R^{n}$ and $\mcV \subset \R^N$ and a locally Lipschitz function $G : \mcU \rightarrow \mcV$ satisfying, for all $s \in \mcU$ $\mathcal{N}(G(s), s) = 0$ with $G$ is path differentiable. Since, by assumption, the solution $\nu(A,b,c)$ to $\mathcal{N}(\nu(A,b,c), A,b,c) = 0$ is unique, $\nu$ coincides with $G$ on $\mathcal{U}$. Thus, $\nu$ is path differentiable and a conservative Jacobian for $\nu$ is given by:
\begin{equation*}
    \cmap_{\nu}(A,b,c) = \left\{-U^{-1}V: [U\ V]\in \cmap_{\mathcal{N}}(\nu(A,b,c),A,b,c)\right\}
\end{equation*}

Let us now turn to $\phi$. Since $P_{\mathcal{K}^*}$ has for conservative Jacobian $J_{P_{\mathcal{K}^*}}$, we may construct a conservative Jacobian for the function $\phi$ as follows using \cite[Lemmas 3, 4, and 5]{bolte2020conservative}:
\begin{align*}
    \cmap_\phi(z) = 
    \begin{bmatrix}
        \Id_n  & 0 \\
        0&\cmap_{P_{\mathcal{K}^*}}(v) \\
        0 &(\cmap_{P_{\mathcal{K}^*}}(v) - \Id_m)
    \end{bmatrix}.
\end{align*}

It follows from Proposition~\ref{prop:jacobianComposition} that the composition $\sol = \phi \circ \nu$ is also path differentiable with conservative Jacobian

\begin{equation*}
   \cmap_{\sol}(A,b,c) = \cmap_\phi(\nu(A,b,c)) \cmap_\nu(A,b,c).
\end{equation*}
\end{proof}

\subsection{Hyperparameter selection for nonsmooth Lasso-type model}\label{app:HO}
\begin{appprop}{\ref{prop:HOprop}}[Conservative Jacobian for the solution mapping]
For all $\lambda\in\R$, assume $X_{\mcE}^TX_{\mcE}$ is invertible where $X_{\mcE}$ is the submatrix of $X$ formed by taking the columns indexed by $\mcE$. Then $\bhat(\lambda)$ is single-valued, path differentiable with conservative Jacobian, $\cmap_{\bhat}\para{\lambda}$,  given for all $\lambda$ as
$$\left\{
\sbrac{-e^{\lambda}\para{\Id_p - \diag\para{q}\para{\Id_p - X^TX}}^{-1} \diag\para{q}\sign\para{\bhat - X^T\para{X\bhat-y}}}\ : q \in \mathcal{M}(\lambda)\right\}
$$
where $\mathcal{M}(\lambda) \subset \RR^p$ is the set of vectors $q$ such that $q_i\in\begin{cases}\brac{1} & i\in\supp \hat \beta \\ [0,1] & i\in\mcE\setminus\supp \hat \beta \\ \brac{0} & i\not\in\mcE \end{cases}$.
\end{appprop}
\begin{proof}
Our goal is to apply Corollary~\ref{cor:implicitFun} to the path differentiable ``optimality gap'' function $F:\R\times\R^{p}\to\R^p$ defined in \eqref{eq:fixedPt}. For each $\lambda\in\R$, the invertibility of $X_{\mcE}^TX_{\mcE}$ guarantees the uniqueness of $\bhat\para{\lambda}$ (see \cite{osborne2000lasso}, \cite[Lemma 1]{mairal2012complexity}), i.e., $\bhat:\R\to\R^p$ is a function. 
Because $\norm{\cdot}{1}$ is separable, the components of the prox can be written, for any $\para{\lambda,u}\in\R\times\R^p$, for all $i\in\brac{1,\ldots,p}$, as
\newq{
[\prox_{e^{\lambda}\norm{\cdot}{1}}\para{u}]_i = \prox_{e^{\lambda}\absv{\cdot}}\para{u_i}
}
which have Clarke subdifferentials
\newq{
\partial^c\prox_{e^{\lambda}\absv{\cdot}}\colon u_i \rightrightarrows \One_{u_i,e^{\lambda}}\quad \times \begin{bmatrix} 1 \\ -\sign(u_i) \end{bmatrix}\quad\mbox{where}\quad\One_{e^{\lambda}}\para{u_i} :=  \begin{cases} 0 & \absv{u_i}< e^{\lambda}\\ [0,1] & \absv{u_i}=e^{\lambda}\\ 1 & \absv{u_i}>e^{\lambda}\end{cases}.
}
Thus a conservative Jacobian for $F$ at $(\lambda,\beta)$ is given by
\nnewq{\label{eq:consJac}
        \cmap_F\colon \para{\lambda,\beta}\rightrightarrows
            \{ [
                 \underbrace{e^{\lambda}\mathrm{diag}(q)
                \sign(\beta - X^T(X \beta - y))}_{A}
            \quad
                \underbrace{\Id_p - \mathrm{diag}(q)\para{\Id_p - X^TX}
            }_{B}] : q \in \mathcal{C}
            \}
}
with $\mathcal{C}:=\{q: q_i\in\One_{e^\lambda}\para{\beta_i - X_{i}^T\para{X\beta - y}}\}$.
Let us estimate the factors $q_i$ above in terms of the equicorrelation set $\mcE$. Recall the KKT conditions \cite{tibshirani2013lasso} for the Lasso problem; a solution $\bhat$ must satisfy
\nnewq{\label{eq:KKT}
X^T\para{y-X\bhat}=e^{\lambda}\delta\quad\quad\mbox{where}\quad\quad\delta_i\in\begin{cases}\brac{\sign\para{\bhat_i}} & i\in\supp \hat \beta \\ [-1,1] & i\not\in\supp \hat \beta \end{cases}.
}
For $i\in\supp \hat \beta $, \eqref{eq:KKT} gives
\newq{
X_i^T\para{y-X\bhat}=e^{\lambda}\sign\para{\bhat_i}&\implies \sign\para{X_i^T\para{y-X\bhat}} = \sign\para{\bhat_i}\\
&\implies \sign\para{\bhat_i} = \sign\para{\bhat_i - X_i^T\para{X\bhat-y}}\\
&\quad\quad\quad\quad\quad\quad\ \ \ = \sign\para{X_i^T\para{y-X\bhat}}.
}
Noting that $\absv{\bhat_i}>0$ and $\absv{X_i^T\para{y-X\bhat}}=e^{\lambda}$ since $i\in\supp \hat \beta \subset\mcE$,
\newq{
\absv{\bhat_i - X_i^T\para{X\bhat-y}} &= \sign\para{\bhat_i - X_i^T\para{X\bhat-y}} \para{\bhat_i - X_i^T\para{X\bhat-y}}\\
&= \sign\para{\bhat_i}\bhat_i + \sign\para{X_i^T\para{y-X\bhat}}X_i^T\para{y-X\bhat}\\
&= \underbrace{\absv{\bhat_i}}_{>0} + \underbrace{\absv{X_i^T\para{y-X\bhat}}}_{=e^{\lambda}}\\
&\implies q_i=1.
}
For $i\not\in\mcE$, $\bhat_i=0$ since $\supp \hat \beta \subset\mcE$. By \eqref{eq:KKT}, we have $\absv{X_i^T\para{y-X\bhat}}\leq e^{\lambda}$. However, since $i\not\in\mcE$, the inequality is strict
\newq{
\absv{X_i^T\para{y-X\bhat}} < e^{\lambda}
}
and can be used to solve for $q_i$
\newq{
\absv{\bhat_i-X_i^T\para{X\bhat-y}} = \absv{X_i^T\para{y-X\bhat}} < e^{\lambda} \implies q_i = 0.
}
Finally, for $i\in\mcE\setminus\supp \hat \beta $, $\bhat_i=0$ and $\absv{X_i^T\para{X\bhat-y}} = e^{\lambda}$ which gives
\newq{
\absv{\bhat_i - X_i^T\para{X\bhat-y}} = \absv{X_i^T\para{X\bhat-y}} = e^{\lambda}
}
and thus $q_i\in[0,1]$. Putting everything together we get an expression for $q_i$ in terms of $\mcE$ and $\supp \hat \beta $
\nnewq{\label{eq:qexpression}
q_i \in\begin{cases} \brac{1} & i\in\supp \hat \beta \\ [0,1] & i\in\mcE\setminus\supp \hat \beta \\ \brac{0} & i\not\in\mcE\end{cases},
}
i.e., $q\in\mathcal{M}$. We proceed to show that $B$ is invertible for all $\lambda\in\R$. Denote $Q:= \diag\para{q}$ for brevity; using the same argument of \cite[Theorem 2]{winston2020monotone} involving similarity transformations and continuity, the matrix $B$ is invertible if and only if
\newq{
\tilde{B} :=  \Id_p - Q^{1/2}\para{\Id_p - X^TX}Q^{1/2} = \Id_p - Q + Q^{1/2}X^TXQ^{1/2}
}
is invertible. Since $\tilde{B}\succeq \Id_p - Q$, it follows that $\lker \para{\tilde{B}} \subset \lker \para{\Id_p-Q}$, however $\lker \para{\Id_p-Q}$ is a subspace of $W_{\mcE}:=  \lspan\brac{e_j: j\in\mcE}$ corresponding to $q_j = 1$. Since $q_j=1\implies j\in\mcE$ by \eqref{eq:qexpression}, the restriction of $\tilde{B}$ to $\lker\para{\Id_p-Q}$ is a principal submatrix of (possibly equal to) $X_{\mcE}^TX_{\mcE}$ which is invertible by assumption. Thus $B$ is invertible and applying Corollary~\ref{cor:implicitFun} then yields the final result.
\end{proof}
\begin{remark}{\rm
Taking $q_i=1$ for all $i\in\mcE$ gives a selection of the conservative Jacobian for $\bhat$ in Proposition~\ref{prop:HOprop}, for all $j\in\brac{1,\ldots,p}$,
\newq{
[\cmap_{\bhat}\para{\lambda}]_j = -e^{\lambda}\sbrac{\para{X_{\mcE}^TX_{\mcE}}^{-1}\sign\para{X_{\mcE}^T\para{y-X\bhat}}}_j \mbox{ if } j\in\mcE, \mbox{ and } [\cmap_{\bhat}\para{\lambda}]_j =0 \mbox{ otherwise}.
}
This corresponds to the directional derivative given by LARS algorithm \cite{efron2004least}, see also \cite{mairal2012complexity}. Alternatively, taking $q_i=0$ for $i\not\in\supp\bhat$ gives, for all $j\in\brac{1,\ldots,p}$, 
$$
[\cmap_{\bhat}(\lambda)]_j=-e^{\lambda}\sbrac{(X_{\supp\bhat}^TX_{\supp\bhat}^{-1})\sign(X^T_{\supp\bhat}(y-X\bhat))}_j, \mbox{ if } j\in \supp\bhat$$
and $[\cmap_{\bhat}(\lambda)]_j =0$ otherwise. This is the weak derivative given by \cite{bertrand2020implicit}. Both of these expressions are particular selections in $J_{\hat{\beta}}$, which is the underlying conservative field. They agree if $\mcE=\supp\bhat$, which holds under qualification assumptions, see for example \cite{bertrand2021implicit} and references therein.}
\end{remark}

\section{Results from Section~\ref{sec:algorithms}}
\label{app:Theorem3}
\begin{apptheorem}{\ref{th:convergenceAlgo}}[Convergence result]
    Consider minimizing $\ell$ given in \eqref{eq:finiteSum} using algorithm \eqref{eq:sgd} under Assumption~\ref{ass:structure}. Assume furthermore the following
    \begin{itemize}
        \item\textbf{Step size:} $\sum_{k=1}^{+\infty} \alpha_k = +\infty$ and $\alpha_k = o(1/\log(k))$.
        \item\textbf{Boundedness:} there exists $M>0$, and $K \subset \RR^p$ open and bounded, such that, for all $s \in (s_{\min},s_{\max})$ and $w_0 \in \mathrm{cl}\  K$, $\|w_k\| \leq M$ almost surely.
    \end{itemize}
   For almost all $w_0 \in K$ and $s\in(s_{\min},s_{\max})$, the objective value $\ell(w_k)$ converges and all accumulation points $\bar{w}$ of $w_k$ are  Clarke-critical in the sense that $0 \in \partialc \ell(\bar{w})$.
\end{apptheorem}
\begin{proof}
We first show that if $w_0$ is taken uniformly at random on $K$ then, almost surely, all iterates $(w_k)_{k \in \NN}$ are random variables which are absolutely continuous with respect to the Lebesgue measure. This is essentially a repeating of the arguments developed in \cite{bianchi2020convergence} for constant step sizes. Assume from now on that $w_0$ is random, uniformly on $K$.

For $i\in\{1,\ldots,N\}$, denoting by $\phi(\cdot,i) \colon \RR^p \to \RR^p$ the output of backpropagation applied to $\ell_i = g_{i,L} \circ g_{i,L-1} \circ \ldots \circ g_{i,1}$, we have that $x \mapsto \phi(x,i)$ is a selection in the conservative Jacobian (actually conservative gradient) $J_i$. Therefore, using \cite[Proposition 1]{bianchi2020convergence} the sequence $(w_k)_{k \in \NN}$ is an SGD sequence in the sense of \cite[Definition 2]{bianchi2020convergence}. 

Compositions of definable functions and functions implicitly defined based on definable functions are definable. Therefore by Assumption~\ref{ass:structure}, for each $i\in\{1,\ldots,N\}$, $\ell_i$ is locally Lipschitz and definable and thus so is $\ell$. Definable functions are twice differentiable almost everywhere so that \cite[Proposition 3]{bianchi2020convergence} applies. Following the recursion argument in \cite[Proposition 2]{bianchi2020convergence}, there exists a set $\Gamma\subset (0,\infty)$ of full Lebesgue measure such that, if $s\alpha_k \in \Gamma$ for all $k \in \NN$, each iterate $(w_k)_{k\in \NN}$ is a random variable which is absolutely continuous with respect to the Lebesgue measure.
We have that
\newq{
\{s\in(s_{\min},s_{\max}): \exists k\in\N, s\alpha_k\in(0,\infty)\backslash\Gamma\} = \bigcup\limits_{k=1}^\infty \{s\in(s_{\min},s_{\max}): s\alpha_k\in(0,\infty)\backslash\Gamma\}
}
is a countable union of null sets and thus a null set, i.e., for almost all $s\in(s_{\min},s_{\max})$, for all $k\in\N$, $s\alpha_k\in\Gamma$. As a result, for almost all $s$, $w_k$ has a density with respect to the Lebesgue measure for all $k\in\N$. 

Conservative gradients are gradients almost everywhere and so there is a full measure set $S$ such that, for all $w \in S$ and all $i\in\{1, \ldots, N\}$, $J_i(w) = \{\nabla \ell_i(w)\}$ \cite[Theorem 1]{bolte2020conservative}. Combining this with the fact that each element of the sequence is absolutely continuous with respect to the Lebesgue measure, the same argument as in \cite[Theorem 1]{bianchi2020convergence} gives, for almost all $s\in(s_{\min},s_{\max})$, for every $k\in\N$, almost surely
\newq{
w_{k+1} = w_k - s\alpha_k \nabla \ell_{I_k} (w_k)
}
 and
\newq{
\mathbb{E} (w_{k+1}| w_0,\ldots,w_k) = w_k - s\alpha_k \nabla \ell (w_k) = w_k -  s\alpha_k \partial^c \ell (w_k).
}
Therefore, the sequence is actually a Clarke stochastic subgradient sequence almost surely (see, for example, \cite{davis2020stochastic}) and thus can be analyzed using the method developed in \cite{benaim2005stochastic}. Indeed, conservativity ensures that $\ell$ is a Lyapunov function for the differential inclusion $\dot{w} \in - \partial^c \ell (w)$, that is decreasing along solutions, strictly outside of $\mathrm{crit}_\ell := \{w \in \RR^p, \, 0 \in\partial^c \ell (w) \}$. Since $\ell$ is definable, the set of its critical values, $\ell(\mathrm{crit}_\ell)$ is finite \cite{bolte2007clarke} and thus has empty interior. By \cite[Theorem 3.6]{benaim2005stochastic} and \cite[Proposition 3.27]{benaim2005stochastic}, it is then guaranteed that $\ell(\bar{w})$ is constant for all accumulation points $\bar{w}$ of $(w_k)_{k\in\N}$ and that $0\in\partial^c \ell(\bar{w})$. This occurs almost surely with respect to the randomness induced by $w_0$ and $(I_k)_{k\in\N}$ and therefore it is true with probability one for almost all $w_0$.
\end{proof}

\section{Results from Section~\ref{sec:pathologies}}\label{sec:auxiliaryResults}

\subsection{Cyclic gradient descent}\label{app:cycle}

\subsubsection{Fixed-point formulation}
\label{app:cycle_cvxopt_fixed_point}
Consider the optimization problem
\begin{equation}
    \label{app:cvxlayer}
    (s_1, s_2) \in \underset{(a,b)\in [0,3]\times[0,5]}{\arg \max} (a+b)( -3x +y + 2 ).
\end{equation}

The optimality condition for this problem can be expressed using the fixed-point equation of the projected gradient descent algorithm. Denote for $x, y \in \R^2$, $q_{x,y} :  (a, b) \mapsto (a + b)(-3x + y + 2)$; we can verify $(s_1, s_2)$ is solution to \eqref{eq:cvxcycledynamic} if and only if it satisfies the equality

\begin{equation*}
    \begin{bmatrix}
        s_1  \\
        s_2  
    \end{bmatrix} = P_{\mcU} \left(    \begin{bmatrix}
        s_1  \\
        s_2  
    \end{bmatrix} + \nabla q_{x, y}(s_1, s_2) \right) = P_{\mcU} \left(     \begin{bmatrix}
        s_1  \\
        s_2  
    \end{bmatrix}  + 
    \begin{bmatrix}
        -3x + y + 2 \\
        -3x + y + 2
    \end{bmatrix} \right).
\end{equation*}

Where $P_{\mcU}$ is the projection on the set $\mcU :=  \left[0, 3\right] \times \left[0, 5\right]$ which can be implemented as a difference of relu functions
\begin{equation*}
    P_{\mcU}(x,y) = \relu(x,y) - \relu(x-3, y-5).
\end{equation*}
Let $h : \R^2 \times \R \times \R \rightarrow \R^2$ be the function
\begin{equation*}
    h : (s,x,y) \mapsto P_{\mathcal{U}} \left(     \begin{bmatrix}
        s_1  \\
        s_2  
    \end{bmatrix}  + 
    \begin{bmatrix}
        -3x + y + 2 \\
        -3x + y + 2
    \end{bmatrix} \right).
\end{equation*}

Then the original problem \eqref{app:cvxlayer} is equivalent to the fixed point equation $s = h(x, y, s)$. Indeed, we can easily verify the solutions $s : \R^2 \rightarrow \R^2$ to \eqref{app:cvxlayer} are

\begin{equation*}
s(x,y) =
\left\lbrace
\begin{array}{ccc}
\left\{(0,0) \right\}  & \mbox{if} & -3x + y + 2 < 0\\
\left\{(3, 5)\right\} & \mbox{if} & -3x + y + 2 > 0\\
\left[0, 3\right] \times \left[0, 5\right] & \mbox{if} & -3x + y + 2 = 0
\end{array}\right.
\end{equation*}

which creates a discontinuity for the function $\ell(\cdot, s(\cdot))$, now expressed as

\begin{equation*}
\ell(x,y,s(x,y)) =
\left\lbrace
\begin{array}{ccc}
x^2 + 4y^2  & \mbox{if} & -3x + y + 2 < 0\\
(x - 3)^2 + 4(y - 5)^2 & \mbox{if} & -3x + y + 2 > 0
\end{array}\right. .
\end{equation*}

\subsubsection{Perturbed experiments}
Perturbed experiments are done on the following perturbed loss function
\label{app:perturbedfunction}
\begin{equation*}
\ell_\varepsilon(x,y,s) = \left(\frac{1}{4} + \varepsilon_1\right)(x-s_1)^2 + (1+\varepsilon_2 )(y-s_2)^2  
\end{equation*}
\begin{equation*}
    s \in s_{\varepsilon}(x,y):=\arg \max \left\{ (a+b)( -(3+ \varepsilon_3)x + y + 2 + \varepsilon_4 ):\, a \in [0,3 - \varepsilon_5], b \in [0,5 - \varepsilon_6] \right\} \nonumber
\end{equation*}
with $\varepsilon_1,\ldots,\varepsilon_6$ the perturbations. In Figure~\ref{fig:jittercycle}, we consider several realizations of independent Gaussian variables $\varepsilon_1,\ldots, \varepsilon_6 \sim \mathcal{N}(0, \sigma^2)$ with $\sigma^2 = 0.05$; despite this added noise, the unwanted dynamics persist.
\subsubsection{Conic canonicalization} 
Let $c \in \R^2$ be a parameter vector and consider the problem
\begin{equation*}
	\begin{array}{ll}
		\underset{x \in \left[0, 3\right] \times \left[0,5 \right]}{\max}  &c^T x.\\
	\end{array}
\end{equation*}
It can be formulated as a cone program (P) and its dual (D):
\begin{center}
\begin{tabular}{p{6.5cm}p{6.5cm}}
	{\begin{equation*}
			\begin{array}{lll}
				\text{(P)}
				&\inf  &c^T x\\
				&\text{subject to} &  Ax+s=b\\
				&  &s\in  \mathcal{K}
			\end{array}
	\end{equation*}}
	&
	{ \begin{equation}
			\label{eq:primalDualapp}
			\begin{array}{lll}
				\text{(D)}&\inf& b^T y\\
				&\text{subject to}& A^T y+c=0\\
				&&y\in  \mathcal{K}^*,
			\end{array}
	\end{equation}}
\end{tabular}
\end{center}
where
\begin{equation*} 
    A = \begin{bmatrix}
        \Id_2 \\
        - \Id_2
    \end{bmatrix}\mbox{ and } \ 
    b = \begin{bmatrix}
        3 \\
        5 \\
        0 \\
        0
    \end{bmatrix}. 
\end{equation*}
Let $(x, y, s)$ be a solution to the cone program \eqref{eq:primalDualapp} where $x$ is the primal variable, $y$ is the dual variable, and $s$ the primal slack variable. Then it follows from \eqref{eq:solconic} that a solution $z$ to $\mathcal{N}(z,c) = 0$ is obtained by $z = (x, y - s)$. For $c = (0, 0)$, the solutions are $x \in \left[0, 3\right] \times \left[0, 5\right]$, $s = b - Ax$, and $y = (0, 0, 0, 0)$, hence the uniqueness assumption for Proposition~\ref{prop:pathdiffconeprog} is not satisfied.

\subsubsection{A chaotic dynamics in \texorpdfstring{$\R^4$}{R\string^4}}
\label{app:billiarddynamic}
We combine two cycles of the previous example into a gradient dynamics in $\R^4$. To perform this, we consider a block-separable sum of the same function where we add a scaling parameter $\eta > 0$:
\begin{equation*}
    g : (x, y, z, w) \mapsto f(x,y) + \eta f(z,w).
\end{equation*}
This will combine the two cycles but the parameter $\eta$ will make one cycle ``faster'' than the other. Projecting the path of the gradient descent on the variables $(y,z)$ we obtain a chaotic dynamics filling the space as the number of iterations increases.
\begin{figure}[H]
\centering
\begin{subfigure}{.31\textwidth}
  \centering
  \includegraphics[width=1\linewidth]{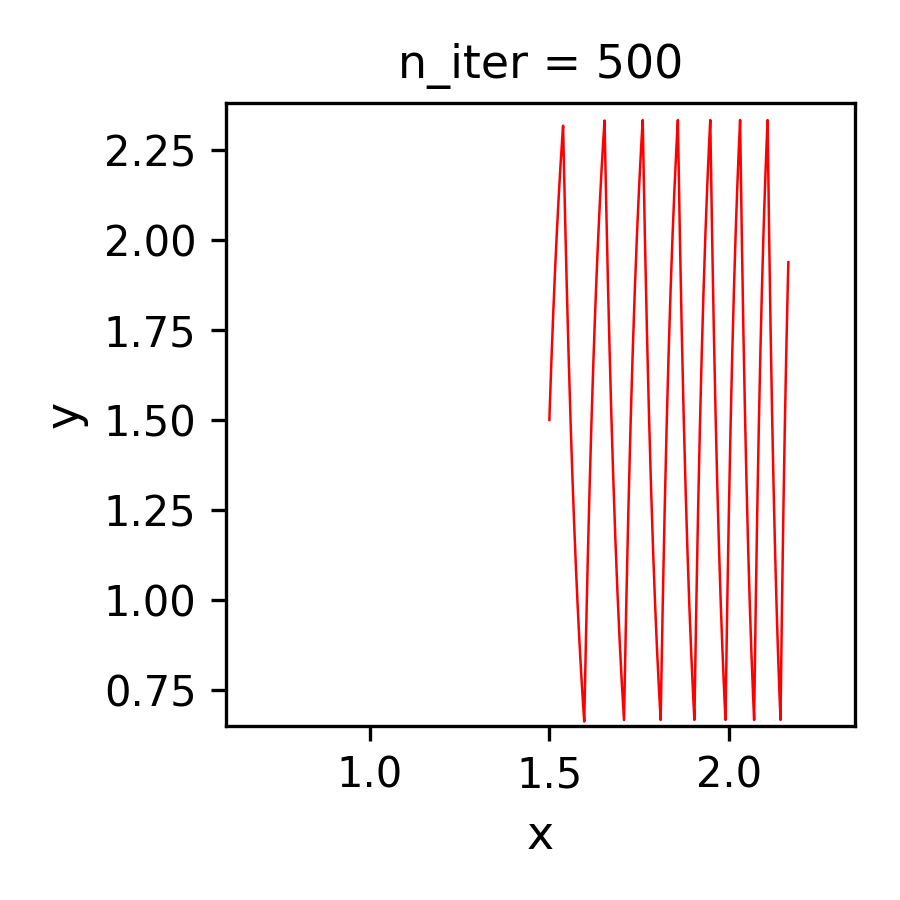}
  \caption{}
\end{subfigure}%
\begin{subfigure}{.31\textwidth}
  \centering
  \includegraphics[width=1\linewidth]{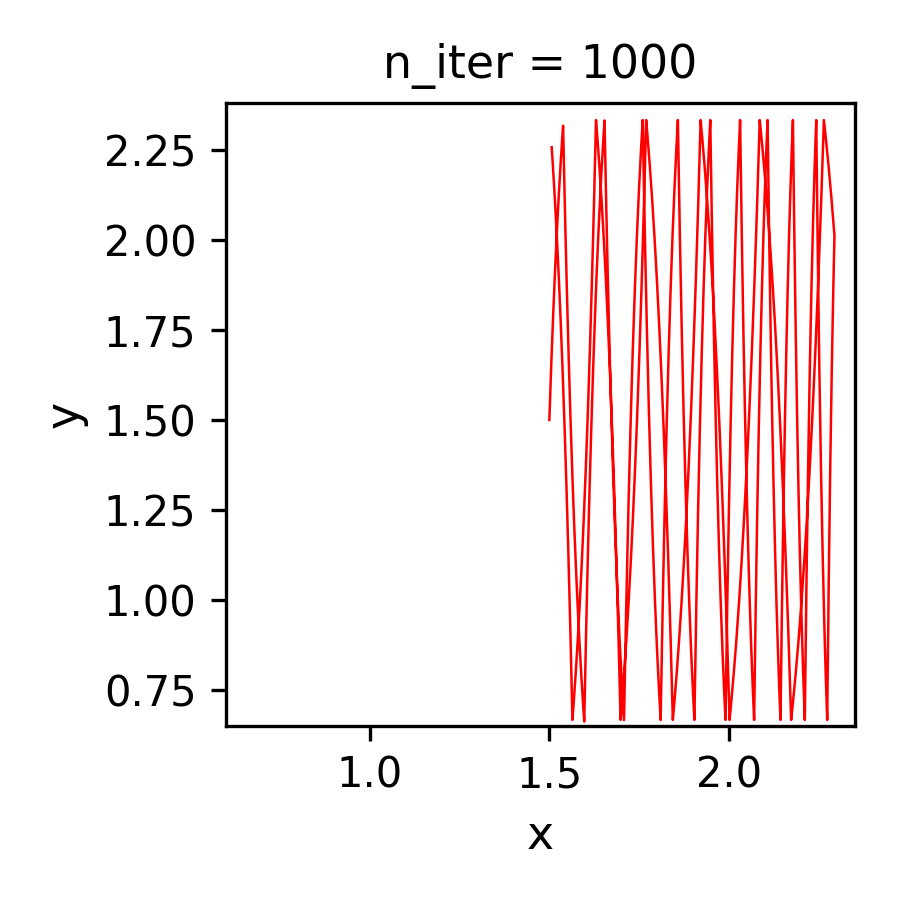}
  \caption{}
\end{subfigure}
\begin{subfigure}{.31\textwidth}
  \centering
  \includegraphics[width=1\linewidth]{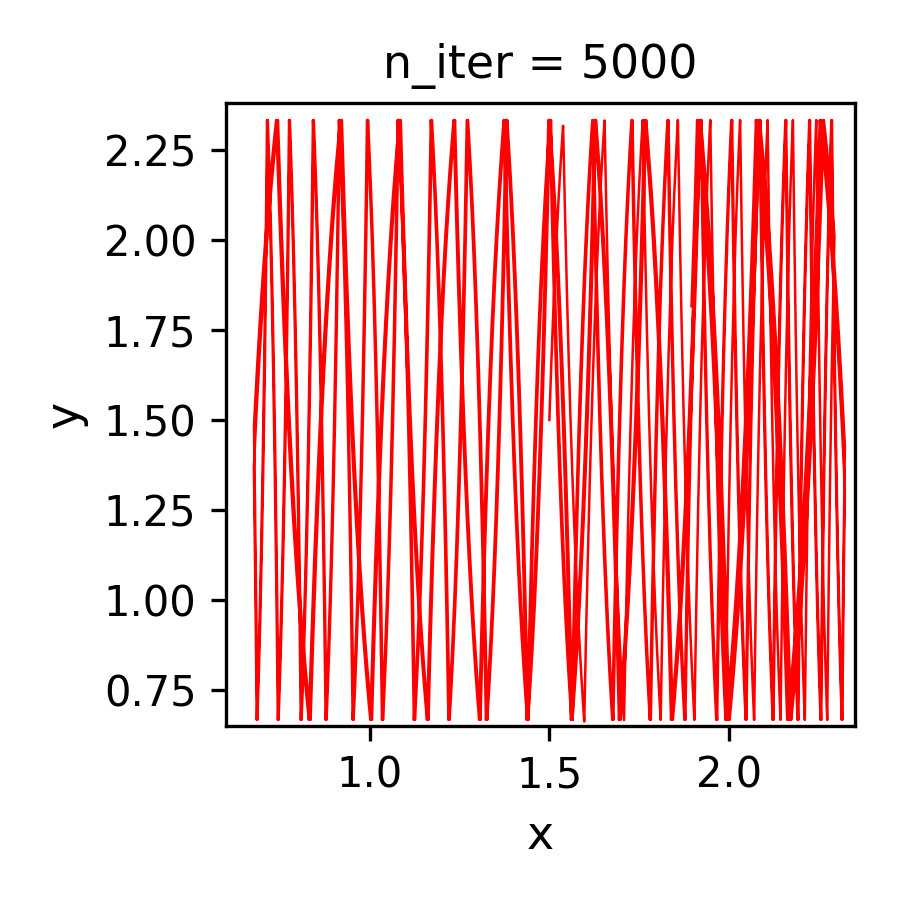}
  \caption{}
\end{subfigure}
\caption{Gradient path after (a) 500, (b) 1000 and (c) 5000 iterations.}
\end{figure}
\subsection{Lorenz-like attractor}
\subsubsection{Objective function is a quadratic form}
\label{app:lorenzquadraticform}
Set $u = (x, y, z)$, then
\begin{align*}
    u^T F(u) &= \sigma x(y - x) +  xy(\rho - z) - y^2 + xyz - \beta z^2 \\
            &= -\sigma x^2 - y^2 - \beta z^2 + (\sigma + \rho)xy\\
            &= \frac{1}{2} u^T H u
\end{align*}
where $H = \begin{bmatrix}
-2\sigma & \sigma + \rho & 0 \\
\sigma + \rho & -2 & 0 \\
0 & 0 & - 2\beta
\end{bmatrix}$. 

For $(\sigma, \rho, \beta) = (10, 28, \frac{8}{3})$, $g$ has for unique critical point $(0, 0, 0)$ which is a strict saddle-point.

\subsection{License of assets used} 
All assets used: cvxpy, cvxpylayers, and JAX were released under the Apache License, Version 2.0, January 2004, \texttt{http://www.apache.org/licenses/}.
\end{document}